\documentclass[lettersize,journal]{IEEEtran}

\usepackage{algorithm}
\usepackage{array}
\usepackage[caption=false,font=normalsize,labelfont=sf,textfont=sf]{subfig}
\usepackage{textcomp}
\usepackage{stfloats}
\usepackage{url}
\usepackage{verbatim}
\usepackage{graphicx}
\usepackage{cite}
\usepackage[utf8]{inputenc} 
\usepackage[T1]{fontenc}    
\usepackage{hyperref}       
\usepackage{url}            
\usepackage{booktabs}       
\usepackage{amsfonts}       
\usepackage{nicefrac}       
\usepackage{microtype}      
\usepackage{xcolor}         
\usepackage{amsmath}
\usepackage{mathtools}
\usepackage{amssymb}
\usepackage{wrapfig}
\usepackage{algpseudocode}
\usepackage{diagbox}
\usepackage{amsthm}
\usepackage{amsbsy}
\usepackage{stmaryrd}

\begin{document}
\newcommand\numberthis{\addtocounter{equation}{1}\tag{\theequation}}


\newcommand{\alg}{\text{ELSS}}


\newcommand{\0}{\mathbb{0}}
\newcommand{\1}{\mathbb{1}}

\newcommand{\E}{\mathbb{E}}
\newcommand{\R}{\mathbb{R}}
\renewcommand{\P}{\mathbb{P}}
\newcommand{\U}{\mathbb{U}}


\newcommand{\ab}{\mathbf{a}}
\newcommand{\bb}{\mathbf{b}}
\newcommand{\eb}{\mathbf{e}}
\newcommand{\ib}{\mathbf{i}}
\newcommand{\pb}{\mathbf{p}}
\newcommand{\qb}{\mathbf{q}}
\newcommand{\vb}{\mathbf{v}}
\newcommand{\ub}{\mathbf{u}}
\newcommand{\xb}{\mathbf{x}}
\newcommand{\yb}{\mathbf{y}}
\newcommand{\zb}{\mathbf{z}}

\newcommand{\Ab}{\mathbf{A}}
\newcommand{\Bb}{\mathbf{B}}
\newcommand{\Cb}{\mathbf{C}}
\newcommand{\Db}{\mathbf{D}}
\newcommand{\Eb}{\mathbf{E}}
\newcommand{\Fb}{\mathbf{F}}
\newcommand{\Gb}{\mathbf{G}}
\newcommand{\Hb}{\mathbf{H}}
\newcommand{\Ib}{\mathbf{I}}
\newcommand{\Jb}{\mathbf{J}}
\newcommand{\Kb}{\mathbf{K}}
\newcommand{\Lb}{\mathbf{L}}
\newcommand{\Pb}{\mathbf{P}}
\newcommand{\Qb}{\mathbf{Q}}
\newcommand{\Rb}{\mathbf{R}}
\newcommand{\Sb}{\mathbf{S}}
\newcommand{\Vb}{\mathbf{V}}
\newcommand{\Ub}{\mathbf{U}}
\newcommand{\Wb}{\mathbf{W}}
\newcommand{\Xb}{\mathbf{X}}
\newcommand{\Yb}{\mathbf{Y}}
\newcommand{\Zb}{\mathbf{Z}}

\newcommand{\alphab}{\boldsymbol{\alpha}}
\newcommand{\betab}{\boldsymbol{\beta}}
\newcommand{\gammab}{\boldsymbol{\gamma}}
\newcommand{\phib}{\boldsymbol{\phi}}
\newcommand{\Phib}{\boldsymbol{\Phi}}
\newcommand{\Qhib}{\boldsymbol{\Qhi}}
\newcommand{\omegab}{\boldsymbol{\omega}}
\newcommand{\psib}{\boldsymbol{\psi}}
\newcommand{\sigmab}{\boldsymbol{\sigma}}
\newcommand{\nub}{\boldsymbol{\nu}}
\newcommand{\thetab}{\boldsymbol{\theta}}
\newcommand{\delb}{\boldsymbol{\delta}}
\newcommand{\rhob}{\boldsymbol{\rho}}
\newcommand{\Pib}{\boldsymbol{\Pi}}
\newcommand{\pib}{\boldsymbol{\pi}}
\newcommand{\Sigmab}{\boldsymbol{\Sigma}}


\newcommand{\Ac}{\mathcal{A}}
\newcommand{\Bc}{\mathcal{B}}
\newcommand{\Cc}{\mathcal{C}}
\newcommand{\Dc}{\mathcal{D}}
\newcommand{\Ec}{\mathcal{E}}
\newcommand{\Fc}{\mathcal{F}}
\newcommand{\Gc}{\mathcal{G}}
\newcommand{\Hc}{\mathcal{H}}
\newcommand{\Lc}{\mathcal{L}}
\newcommand{\Nc}{\mathcal{N}}
\newcommand{\Oc}{\mathcal{O}}
\newcommand{\Pc}{\mathcal{P}}
\newcommand{\Rc}{\mathcal{R}}
\newcommand{\Sc}{\mathcal{S}}
\newcommand{\Tc}{\mathcal{T}}
\newcommand{\Uc}{\mathcal{U}}
\newcommand{\Vc}{\mathcal{V}}
\newcommand{\Xc}{\mathcal{X}}
\newcommand{\Yc}{\mathcal{Y}}
\newcommand{\Zc}{\mathcal{Z}}


\newcommand{\tauh}{\widehat{\tau}}
\newcommand{\Sigmah}{\widehat{\Sigma}}

\newcommand{\fh}{\widehat{f}}
\newcommand{\gh}{\widehat{g}}
\newcommand{\kh}{\widehat{k}}
\newcommand{\qh}{\widehat{q}}
\newcommand{\Rh}{\widehat{R}}


\newcommand{\alphabh}{\widehat{\boldsymbol{\alpha}}}
\newcommand{\thetabh}{\widehat{\boldsymbol{\theta}}}

\newcommand{\qbh}{\widehat{\mathbf{q}}}

\newcommand{\Kbh}{\widehat{\mathbf{K}}}


\newcommand{\Fch}{\widehat{\mathcal{F}}}


\newcommand{\argmin}{\mathrm{argmin}}
\newcommand{\arginf}{\mathrm{arginf}}
\newcommand{\argmax}{\mathrm{argmax}}
\newcommand{\minimize}{\mathrm{minimize}}
\newcommand{\maximize}{\mathrm{maximize}}
\newcommand{\supp}{\mathrm{supp}}


\newcommand{\TV}{\text{TV}}
\newcommand{\norm}[1]{\left\lVert#1\right\rVert}
\newcommand{\tr}[1]{\text{Tr}\left[#1\right]}
\newcommand{\inn}[1]{\left<#1\right>}
\newcommand{\seal}[1]{\left \lceil #1\right \rceil}
\newcommand{\floor}[1]{\left \lfloor #1\right \rfloor}
\newcommand{\abs}[1]{\left|#1\right|}
\newcommand{\ind}[1]{\mathbf{1}\left(#1\right)}
\newcommand{\ex}[1]{\E\left[#1\right]}


\newtheorem{theorem}{Theorem}
\newtheorem{acknowledgement}[theorem]{Acknowledgement}
\newtheorem{assumption}{Assumption}
\newtheorem{conjecture}[theorem]{Conjecture}
\newtheorem{corollary}[theorem]{Corollary}
\newtheorem{definition}{Definition}
\newtheorem{example}{Example}
\newtheorem{lemma}[theorem]{Lemma}
\newtheorem{fact}{Fact}
\newtheorem{problem}{Problem}
\newtheorem{proposition}[theorem]{Proposition}
\newtheorem{remark}{Remark}
\newtheorem{solution}[theorem]{Solution}
\newtheorem{summary}[theorem]{Summary}

\newcommand{\bl}{\color{blue}}
\newcommand{\rd}{\color{black}}

\title{Theoretical Analysis of Measure Consistency Regularization for Partially Observed Data}

\author{Yinsong Wang,
        and~Shahin~Shahrampour,~\IEEEmembership{Senior Member,~IEEE}
\IEEEcompsocitemizethanks{
\IEEEcompsocthanksitem Y. Wang is with the University of Nevada Reno, Reno, NV 89557 USA.\\
E-mail: {\tt\small yinsongw@unr.edu}
\IEEEcompsocthanksitem S. Shahrampour is with the Department of Mechanical and Industrial Engineering at Northeastern University, Boston, MA 02115 USA.\protect\\
E-mail: {\tt\small s.shahrampour@northeastern.edu}
}}

\IEEEtitleabstractindextext{%
\begin{abstract}

The problem of corrupted data, missing features, or missing modalities continues to plague the modern machine learning landscape. To address this issue, a class of regularization methods that enforce consistency between imputed and fully observed data has emerged as a promising approach for improving model generalization, particularly in partially observed settings. We refer to this class of methods as \textbf{Measure Consistency Regularization (MCR)}. Despite its empirical success in various applications, such as image inpainting, data imputation and semi-supervised learning, a fundamental understanding of the theoretical underpinnings of MCR remains limited. This paper bridges this gap by offering theoretical insights into {\color{black} when MCR yields a more favorable finite-sample estimation-error upper bound}, viewed through the lens of neural network distance.

{\color{black} Under ideal interpolation and compatibility conditions, we show that the MCR estimation-error upper bound is no larger than vanilla supervised training and becomes strictly smaller when the mixed-sample alternative is favorable. We then extend the analysis to the non-ideal regime, where optimization and compatibility residuals can potentially offset this finite-sample advantage. Guided by these insights, we propose a novel practical diagnostic that leverages the duality gap and a calibrated drift estimator to infer the potential benefit of MCR training. We present detailed empirical evidence to support our theoretical claims and to show the effectiveness and accuracy of our practical diagnostic.} We further provide simulations on real-world datasets to show the versatility of MCR under different model architectures designed for different data sources.

\end{abstract}

\begin{IEEEkeywords}
Imputation, Integral Probability Metric, Neural Net Distance, Measure Consistency Regularization
\end{IEEEkeywords}}

\maketitle

\IEEEdisplaynontitleabstractindextext

\IEEEpeerreviewmaketitle

\section{Introduction}

Rapid advances in data acquisition technologies, ranging from high-resolution cameras and environmental sensors to transmission electron microscopy and single-cell sequencing, have fueled an unprecedented boom in data generation, forming the backbone of modern machine learning and artificial intelligence. Yet, despite this explosion in data volume, a persistent and troubling issue remains: missing data, in which features are unobserved or lost. This problem is pervasive across domains and has become even more pronounced as data have grown richer and more multimodal. In single-cell biology, for example, the historical norm was to profile each cell using only one omic technology, such as gene expression, chromatin accessibility, or protein abundance, leaving the remaining modalities unobserved—effectively creating missing features. Truly multimodal single-cell datasets have only begun to emerge in recent years \cite{swanson2021simultaneous, mimitou2021scalable}, underscoring how deeply entrenched this limitation has been. Moreover, the issue is not confined to specialized fields: more than $40\%$ of datasets in the UCI Machine Learning Repository contain missing values \cite{choudhury2020missing}, emphasizing the ubiquity of missingness even in widely used benchmark datasets.

The most direct approach to handling missing data is imputation. Formally, the problem can be described using two datasets: (i) a {\it fully} observed dataset, $\Db_l={\{(\xb_i, \zb_i)\}}_{i=1}^n$, consisting of $n$ data points where $\xb \in \Xc \subseteq \R^{d_x}$ and $\zb \in \Zc \subseteq \R^{d_z}$ are two sets of features; and (ii) a {\it partially} observed dataset, $\Db_u=\{\xb_i\}_{i=n+1}^{n+m}$, of $m$ data points for which $\zb$ is missing. The imputation goal is to construct a prediction function $f$ that can recover or approximate the unobserved features. A wide range of methods have historically been proposed for this task, often tailored to specific missing-data patterns. Classical examples include K-nearest neighbors (KNN) imputation \cite{garcia2009k}, regression-based approaches \cite{yu2020regression}, and multiple regression techniques \cite{zhao2016multiple}, among others.

The recent explosion in both data volume and dimensionality has motivated various research communities to develop domain-specific solutions to the missing data imputation problem. Examples span computer vision \cite{tran2017missing, ma2021smil, joy2021learning, sun2024redcore}, sensor networks \cite{kang2023cm, cheng2024collaboratively}, and single-cell data integration \cite{ashuach2023multivi, tu2022cross, gong2021cobolt}. Although these efforts share a common underlying mathematical formulation, the problem often appears under different names, such as image inpainting \cite{elharrouss2020image}, cross-modal imputation \cite{cohen2023joint}, imputation with structured missingness \cite{mitra2023learning}, RNA imputation \cite{bahador2021reconstruction}, or missing channel reconstruction \cite{hou2020systematic}. A unifying characteristic across much of this literature, however, is that the partially observed dataset $\Db_u$ is typically leveraged only as an auxiliary resource for model testing or to support downstream learning tasks after imputation, rather than being directly employed to improve the imputation quality itself. 

Interestingly, a practice that has emerged across different domains explicitly leverages the partially observed dataset $\Db_u$ to enhance imputation quality. This simple, practical approach lacks a formal name or a theoretical foundation despite its prevalence across domains. The central idea is intuitive: to improve imputation, one can impose a regularization constraint requiring that the empirical measure of the imputed dataset derived from $\Db_u$ remains consistent with that of the fully observed dataset $\Db_l$. In this work, we refer to this strategy as \textit{measure consistency regularization} ({\bf MCR}), which we will formally introduce in Section \ref{sec:pre}. Although this idea has achieved notable empirical success in many applications \cite{yoon2020gamin, yoon2018gain, muzellec2020missing, wu2023jointly}, its theoretical properties have remained elusive, leaving open fundamental questions about why, when, and how it works.

To address this theoretical gap, we undertake a comprehensive analysis of MCR through the lens of the \textit{neural net distance} \textbf{(NND)} \cite{pmlr-v70-arora17a}, due to its generality and theoretical properties. NND represents a broad class of discrepancy measures that can serve as effective consistency regularizers. Building on this foundation, we summarize our main contributions as follows:
\begin{itemize}
    \item We provide a unified view on imposing MCR between imputed data and fully observed data by incorporating NND as a penalty function in the learning objective. This formulation generalizes several existing approaches \cite{yoon2018gain,taherkhani2021self} that demonstrated strong empirical effectiveness.
    \item  {\color{black}
We establish an estimation-error bound for NND-based MCR in Theorem~\ref{the:main}. Under explicit interpolation and compatibility conditions, MCR's estimation error bound could be tighter than that of imputation without MCR, although its asymptotic order remains $\mathcal{O}(n^{-1/2})$.
    \item We extend this estimation-error analysis to the non-ideal scenario in Theorem \ref{the:imperfect}, where the training loss converges to a small but non-zero value, and discrepancies exist between fully and partially observed data. Our results reveal that the upper-bound advantage of MCR training is not always guaranteed in practice. Based on this analysis, we derive an exact duality-gap-based sufficient condition for an advantageous MCR upper bound and instantiate it as a practical plug-in training diagnostic.
    \item We conduct extensive simulations to evaluate multiple aspects of our theoretical analysis and the proposed training protocol. We further compare variants of MCR with state-of-the-art imputation algorithms \cite{yoon2018gain, muzellec2020missing} across five real-data tasks, and the empirical results support the versatility of the MCR framework.}
\end{itemize}

\subsection*{Preliminaries and Notations}
We denote the true joint probability measure of the fully observed dataset by $\pib(\xb,\zb)$, or simply $\pib$. The joint measure induced by a prediction function $f$ is denoted as $\pib(\xb,f(\xb))$, abbreviated as $\pib^f$ to emphasize the dependence on $f$. The empirical joint measure of the fully observed dataset is denoted by $\pib_l$, while the empirical joint measure of the partially observed data (imputed by $f$) is denoted by $\pib_u^f$. We further denote by $\pib_{ul}^f$ the empirical measure of the mixture dataset when it is imputed by a function $f$. Specifically, we have
\begin{equation*}
\begin{aligned}
    \pib_l &:= \frac{1}{n} \sum_{i=1}^{n} \delta_{(\xb_i, \zb_i)}\\
    \pib_u^f &:= \frac{1}{m}\sum_{i=n+1}^{n+m} \delta_{(\xb_i, f(\xb_i))}\\
    \pib_{ul}^f &:= \frac{1}{n+m}\sum_{i=1}^{n+m} \delta_{(\xb_i, f(\xb_i))}.
\end{aligned}
\end{equation*}
 For any probability measure $\pib$ on $(\Xc,\Zc)$ and any integrable function $g$, we write $\pib g := \E_{\pib}[g]$, where $\E[\cdot]$ denotes the expectation operator. {\color{black} Table~\ref{tab:notation} presents a notation summary for this work.}

\begin{table}[t]
\centering
\caption{Table of notations.}
\label{tab:notation}
\footnotesize
\setlength{\tabcolsep}{3pt}
\renewcommand{\arraystretch}{1.08}
\begin{tabular}{@{}p{0.32\columnwidth}p{0.62\columnwidth}@{}}
\toprule
\textbf{Symbol} & \textbf{Description} \\
\midrule
$\Db_l,\Db_u;\ n,m$
& Fully and partially observed datasets and their respective sizes. \\

$\pib,\pib^f$
& True joint measure and the joint measure induced by predictor $f$. \\

$\pib_l,\pib_u^f,\pib_{ul}^f$
& Empirical measures for fully observed, imputed partially observed, and imputed mixture datasets. \\

$\Fc,\Gc_{nn},\Hc$
& Function classes for the predictor, neural net distance, and composed functions. \\

$d_{\Gc_{nn}}(\cdot,\cdot)$
& Neural net distance induced by $\Gc_{nn}$. \\

$R(f)$
& Population NND, i.e., $d_{\Gc_{nn}}(\pib,\pib^f)$. \\

$\Lc,\epsilon_{\Lc}$
& Supervised loss and its empirical training value. \\

$\hat{\mathfrak{R}}_N(\cdot)$
& Empirical Rademacher complexity on $N$ samples. \\

$\epsilon_d$
& Optimality gap of the NND penalty. \\

$\xi,\hat{\xi}$
& Distribution-compatibility discrepancy and its empirical estimate. \\

$\Delta,\hat{\Delta}$
& Concentration constant and its empirical estimate. \\

$DG(f,g)$
& Duality gap used to upper-bound $\epsilon_d$. \\
\bottomrule
\end{tabular}
\end{table}

\begin{figure*}
    \centering
    \includegraphics[width=0.9\linewidth]{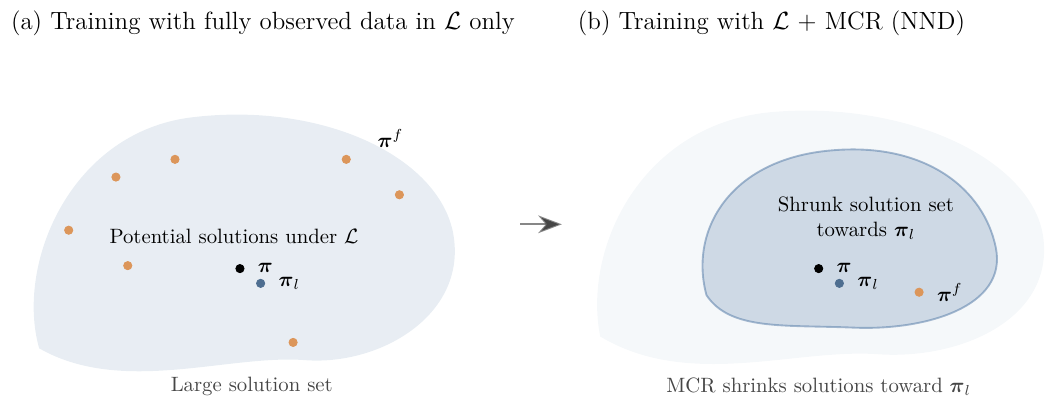}
    \caption{An intuitive illustration of the effect of measure consistency regularization (MCR). 
    When MCR is absent, minimizing the supervised loss based on the empirical measure $\pib_l$ of fully observed data can admit a large set of solutions, making the imputed measure induced by the learned function more likely to deviate from the true measure $\pib$. By contrast, enforcing MCR encourages consistency between the imputed measure and $\pib_l$, narrowing the solution set and, in turn, increasing the likelihood that the imputed measure remains close to the true measure.
}
    \label{fig:overview}
\end{figure*}

\section{Measure Consistency Regularization with Neural Net Distance}\label{sec:pre}

\subsection{Imputation without Regularization}

Imputation without regularization is essentially a supervised learning task to learn a prediction function $f^*: \Xc \rightarrow \Zc$ such that
\begin{equation*}
    f^* \in \underset{f \in \Fc}{\arginf} \left\{ \pib \Lc^f\right\},
\end{equation*}
where $\Lc^f$ denotes the loss $\Lc$ evaluated on the predictions of $f$. Because the true distribution $\pib$ is unavailable in practice, the common approach is empirical risk minimization (ERM), which yields the proxy solution
\begin{equation}\label{eq:objective}
    f^*_l \in \underset{f \in \Fc}{\arginf}\left\{ \pib_l \Lc^f \right\},
\end{equation}
based on the empirical distribution $\pib_l$ of fully observed data.  

The distinction between imputation and standard supervised learning lies in the availability of a partially observed dataset $\Db_u$. In much of the imputation literature, $\Db_u$ has traditionally been treated like a test set, i.e., it is not directly used to train $f$, and methodological advances have primarily focused on tailoring the function class $\Fc$ to specific structures while adhering to the ERM framework in Eq.~\eqref{eq:objective} \cite{tran2017missing,smieja2018processing, tu2022cross,sun2024redcore}. For example, convolutional filters are often incorporated for image data, and attention-based Transformer blocks are often used for language data.  

{\color{black}
The setting is also closely related to domain adaptation and semi-supervised learning (SSL). Unlike domain adaptation, we assume that the $\xb$ components of $\Db_l$ and $\Db_u$ follow the same distribution. On the other hand, in SSL, $\zb$ commonly represents an unobserved categorical label, although structured-output extensions also exist \cite{van2020survey}. Popular entropy and consistency regularizers often constrain the prediction confidence or the pointwise stability on $\Db_u$ \cite{berthelot2019mixmatch, sohn2020fixmatch,huang2024flatmatch}; they do not generally align $\pib_l$ and $\pib_u^f$. Instead, MCR constrains these empirical measures to be close, and therefore, it applies directly to continuous, vector-valued, and multimodal targets. Thus, MCR belongs to the broader SSL paradigm but represents a distinct measure-level regularization principle.
}

\subsection{Measure Consistency Regularization}
Intuitively, if a prediction function $f$ is well learned, the distribution of the imputed version of the partially observed dataset $\Db_u$ should resemble that of the fully observed dataset $\Db_l$. To promote this similarity, we can leverage a regularization mechanism that enforces consistency between the empirical measures of $\Db_l$ and the imputed version of $\Db_u$. We refer to this approach as \textit{measure consistency regularization} ({\bf MCR}). Formally, MCR augments the ERM objective in Eq.~\eqref{eq:objective} with an additional regularization term as
\begin{equation}\label{eq:MCR}
    f^*_u \in \underset{f \in \Fc}{\arginf} \left\{ \pib_l \Lc^f + \lambda_d r(\pib_l,\pib_u^f)\right\},
\end{equation}
where $r(\cdot,\cdot)$ denotes a discrepancy measure between two probability distributions, and $\lambda_d$ is a tuning parameter controlling the strength of regularization. Note that the subscript in $f^*_u$ emphasizes the use of the partially observed data $\Db_u$ in training. Under specific loss and discrepancy measure designs, the general form \eqref{eq:MCR} recovers several empirically successful training schemes for imputation \cite{yoon2018gain,muzellec2020missing} and semi-supervised learning \cite{taherkhani2021self,wang2019semi}.

Discrepancy measures between probability distributions are most commonly drawn from two broad families: integral probability metrics (IPMs) and \(f\)-divergences. Notable IPMs include the Wasserstein distance, maximum mean discrepancy (MMD), and total variation distance, while prominent \(f\)-divergences include the Kullback--Leibler divergence, Jensen--Shannon divergence, and squared Hellinger distance. 

Fig. \ref{fig:overview} provides a high-level overview of the intuition behind MCR. Without MCR, minimizing the supervised loss based on the  empirical measure $\pib_l$ of fully observed data admits a large set of solutions, making the imputed measure induced by the learned function likely to deviate from the true measure $\pib$. By enforcing MCR, the learning process encourages the imputed measure to align with $\pib_l$, thereby shrinking the solution set and increasing the likelihood that the imputed measure remains close to the true measure.

\subsection{Neural Net Distance for MCR}
In this work, we focus on the IPM family for measure consistency regularization. The IPM defines a discrepancy between two probability measures $\mu$ and $\nu$ with respect to a function class $\Gc$ as
\begin{equation*}
    \texttt{IPM}_{\Gc}(\mu,\nu) := \underset{g \in \Gc}{\sup} \; \big|\mu g - \nu g\big|.
\end{equation*}
Several well-known probability metrics are recovered as special cases of IPM: kernel maximum mean discrepancy (MMD) when $\Gc$ is a reproducing kernel Hilbert space (RKHS) \cite{ASENS_1953_3_70_3_267_0}, the $1-$Wasserstein distance when $\Gc$ is the set of $1-$Lipschitz functions \cite{kantorovich1958space}, and the total variation distance when $\Gc$ consists of indicator functions \cite[Chapter~3]{rachev2013methods}.  

Compared to the other widely used class of \(f\)-divergences, IPMs possess several desirable properties. They are symmetric \cite{villani2008optimal} and robust under support mismatch \cite{arjovsky2017wasserstein}. Moreover, many IPMs admit estimators with linear computational complexity in the number of samples. This contrasts with certain optimal transport metrics closely related to IPMs, such as the quadratic-cost Wasserstein distance \cite{zhao2023transformed}, Sinkhorn distance \cite{muzellec2020missing}, and masked optimal transport (MOT) \cite{wu2023jointly}, whose computational costs typically scale quadratically with sample size.  

Despite these advantages, direct computation of IPMs can be difficult because the function class $\Gc$ rarely yields analytical solutions. Neural net distance (NND) provides a practical surrogate by restricting $\Gc$ to neural network architectures \cite{pmlr-v70-arora17a}. Specifically, the discrepancy between measures $\mu$ and $\nu$ is defined as
\begin{equation}\label{dgnn}
    d_{\Gc_{nn}}(\mu,\nu) := \underset{g \in \Gc_{nn}}{\sup}\; \big|\mu g - \nu g\big|,
\end{equation}
where $\Gc_{nn}$ denotes the class of neural networks. One can show that the above definition automatically satisfies the properties of metrics, such as non-negativity, symmetry, and triangle inequality. Moreover, it is discriminative, i.e., $d_{\Gc_{nn}}(\mu,\nu)=0$ if and only if $\mu=\nu$ for sufficiently rich $\Gc_{nn}$ (e.g., RKHS unit ball with characteristic kernels, and Lipschitz functions) \cite{zhang2017discrimination}, therefore making it a proper distance metric as well. The maximization problem in Eq.~\eqref{dgnn} is typically solved with gradient-based optimization. In the context of Generative Adversarial Networks (GANs), the function $g$ is often interpreted as the discriminator \cite{arjovsky2017wasserstein, gulrajani2017improved}.  

The neural net distance thus defines an optimization-friendly metric that integrates seamlessly into the MCR framework in Eq.~\eqref{eq:MCR}. The resulting formulation is
\begin{equation}\label{eq:ndp}
    f^*_u \in \underset{f \in \Fc}{\arginf} \left\{ \pib_l \Lc^f + \lambda_d d_{\Gc_{nn}}(\pib_l, \pib^f_u)\right\},
\end{equation}
where $d_{\Gc_{nn}}$ is defined in Eq.~\eqref{dgnn} with $g \in \Gc_{nn}: \Xc \times \Zc \rightarrow \R$, and $\lambda_d$ controls the strength of the regularization.  

In the remainder of this work, we develop our theoretical analysis specifically for MCR with neural net distance, owing to its generality, flexibility, and practical tractability.

\section{Literature Review} \label{sec:literature}

\textbf{I) Measure Consistency Regularization:} There exist several strands of literature that employ the idea of enforcing measure consistency between the imputed data and observed data. For example, \cite{yoon2018gain} proposes to use discriminators to differentiate the observed and imputed portion of each data point. \cite{yoon2020gamin} further introduces observed data supervision to the work of \cite{yoon2018gain}. \cite{muzellec2020missing,wu2023jointly, zhao2023transformed} use second-order optimal transport to enforce the measure consistency. 

{\color{black} MCR is also related to SSL regularization. Many prominent SSL methods use entropy minimization, pseudo-labeling, or augmentation consistency for categorical outputs \cite{berthelot2019mixmatch,sohn2020fixmatch, huang2024flatmatch}. These objectives constrain individual predictions but do not generally compare the labeled and imputed joint measures. Distribution alignment methods based on optimal transport or IPMs \cite{taherkhani2020transporting,taherkhani2021self}, in contrast, can be viewed as classification-oriented instances of the MCR framework. Related distributionally robust SSL formulations use Wasserstein ambiguity sets to analyze learning from unlabeled data under distributional perturbations \cite{najafi2019robustness}; unlike MCR, their objective is to achieve worst-case robustness rather than direct agreement between the observed and imputed empirical measures.

Our objective is not to establish MCR as universally preferable to other SSL frameworks. These objectives encode different structural assumptions, and a unified comparison would require separate population, finite-sample, and optimization analyses. We leave such a comparison as an interesting direction for future work. Most importantly, despite substantial theoretical work on related SSL and distributionally robust formulations, the finite-sample role of the specific observed-versus-imputed measure-consistency objective considered here remains less understood.}

\textbf{II) Neural Net Distance:} Neural net distance, introduced in \cite{pmlr-v70-arora17a}, is an approximation to integral probability metrics (IPMs), where one uses a class of neural networks as discriminators to measure discrepancy between two distributions. Since then, it has been used extensively in generative adversarial frameworks \cite{arjovsky2017wasserstein, gulrajani2017improved}. Theoretical analyses have examined (i) the convergence of empirical neural-net distances to the true population IPM under capacity constraints \cite{pmlr-v70-arora17a, ji2018minimax}, and (ii) generalization in GANs---i.e.\ how a small neural‐net distance between generated and true distributions implies closeness in some other statistical sense \cite{ji2021understanding, liang2017well, zhang2017discrimination}.

Beyond generation, neural-net distance (and more generally, IPMs) are natural regularizers for imputation/predictive tasks\cite{long2018conditional, li2020maximum}: when imputing missing features, enforcing closeness between the imputed and fully observed data distributions under a neural-net distance can help promote plausible imputations. While many works use adversarial losses (GANs) or discriminators in this spirit (e.g., GAIN \cite{yoon2018gain}), few have explicitly used neural-net or IPM \textit{theory} to understand when such regularization helps in imputation. 

In light of these, our contribution leverages neural-net distance / IPM‐style regularization as part of a \textit{unifying, domain-agnostic} theoretical framework for measure consistency regularization. We aim to show under what conditions such regularizers yield a tighter finite-sample estimation-error upper bound, rather than focusing solely on empirical evidence, as many prior works do, and how this applies across different modalities (tabular, image, multi‐modal, etc.).

\textbf{III) Imputation in Modern Landscape:} Our target problem formulation is in line with imputation tasks under structured missingness \cite{mitra2023learning,cheng2024collaboratively} from a mathematical standpoint. It is more commonly referred to as cross-modal or multimodal imputation \cite{tran2017missing, ma2021smil,wu2023jointly, kang2023cm, sun2024redcore} in recent literature, where the missingness happens to a specific data modality. Single-cell data analysis is one of the most active communities studying imputation problems, mainly due to profiling technology limitations. Recent advances in single-cell data imputation have focused on the modification of the latent space in variational autoencoders \cite{tu2022cross,cohen2023joint,heumos2023best,gayoso2021joint,gong2021cobolt, ashuach2023multivi}. However, all of these techniques only use information from the fully observed dataset in learning the prediction function, and they do not take advantage of the partially observed dataset. 

We note that there exist a number of studies that successfully incorporate the partially observed data with measure consistency regularization. \cite{yoon2018gain,yoon2020gamin} use discriminators to force the generators to impute believable data for partially observed data. \cite{muzellec2020missing, wu2023jointly, zhao2023transformed} apply optimal transport to enforce measure consistency between observed data and imputed data, the complexity of which scales quadratically with respect to the number of fully observed and partially observed data.

\section{Finite-Sample Estimation-Error Analysis of Measure Consistency Regularization}\label{sec:EstTheory}

We analyze the estimation error strictly in terms of NND, a common choice for literature in domain adaptation and generative modeling \cite{ji2021understanding, pmlr-v70-arora17a, zhang2017discrimination, liang2017well, ji2018minimax, zhang2012generalization}. Specifically, we define the generalization error $R(f)$ of a prediction function $f: \Xc \rightarrow \Zc$ with respect to the NND function class $g \in \Gc_{nn}$ as
\begin{equation}\label{eq:risk}
    R(f) := d_{\Gc_{nn}}(\pib,\pib^f) = \underset{g \in \Gc_{nn}}{\sup}\; \big|\pib g - \pib^f g\big|.
\end{equation}

{\color{black}
In the well-specified deterministic setting, this population criterion is identifying. Specifically, if $\zb=f_0(\xb)$ almost surely and $\Gc_{nn}$ is discriminative, then $R(f)=0\Longleftrightarrow\pib=\pib^f \Longleftrightarrow f(\xb)=f_0(\xb)$ almost surely. Thus, the joint-measure agreement constitutes a meaningful population target for imputation. A detailed explanation on this property is provided in the supplementary material.
}

 Since MCR does not alter the architecture of the prediction model (i.e., the function class $\Fc$), the best achievable generalization error is always ${\inf}_{f \in \Fc} d_{\Gc_{nn}}(\pib,\pib^f)$ \cite{ji2021understanding}. Consequently, the key quantity of interest is the {\it estimation error}, defined for a learned function $\hat{f}$ as
\begin{equation}\label{eq:est}
    R(\hat{f}) - \underset{f \in \Fc}{\inf}\;R(f) 
    = d_{\Gc_{nn}}(\pib,\pib^{\hat{f}}) - \underset{f \in \Fc}{\inf} \; d_{\Gc_{nn}}(\pib,\pib^f),
\end{equation}
which quantifies the gap MCR seeks to reduce.

\subsection{Ideal Regime}

We first derive the estimation error bound under ideal regime, where the learned functions exactly minimize the corresponding loss functions in Eq.~\eqref{eq:objective} and Eq.~\eqref{eq:ndp}, {\color{black} and the partially observed data is compatible with the fully observed data in the sense of Assumption~\ref{assump:4}}. We impose the following assumptions.  

\begin{assumption}\label{assump:1}
The diameters of feature space $\Xc$, target space $\Zc$, and their joint space, are bounded with $B_x$, $B_z$ and $B_{xz}$, respectively, with respect to $L_2$-norm. Accordingly, prediction functions $f:\Xc \to \Zc$ are bounded by $B_z$. In addition, all functions $g \in \Gc_{nn}$ are assumed to be Lipschitz continuous, with a constant at most $B_g$, within a symmetric class $\Gc_{nn}$ ($g \in \Gc_{nn} \Leftrightarrow -g \in \Gc_{nn}$). 
\end{assumption}

 We note that the above assumption applies to almost any neural network architecture (with bounded weights after training \cite{fazlyab2019efficient}) and finite datasets \cite{akidau2015dataflow}. We next introduce the
interpolation and compatibility conditions used in the
ideal regime analysis.

\begin{assumption}\label{assump:2}
Let $\Fc_1 := \{f: f\in\arginf_{f \in \Fc} \pib_l \Lc^f\}$ denote the minimizer set of the prediction loss on the fully observed data, and let  $\Fc_2 := \{f: f\in\arginf_{f \in \Fc} d_{\Gc_{nn}}(\pib_l, \pib^f_l)\}$ be the minimizer set of the NND loss on the fully observed data. We assume perfect fit exists on $\Db_l$, such that $\Fc_1 = \Fc_2 = \Fc^*$, where $\Fc^* := \{f: \zb = f(\xb), \; \forall (\xb,\zb) \in \Db_l\}$.
\end{assumption}

{\color{black}
\begin{assumption}\label{assump:3}
Let $\Fc_3
:= \{f: f\in\arginf_{f \in \Fc} d_{\Gc_{nn}}\!\left(\pib_l,\pib_u^f\right)\}$ 
denote the minimizer set of the NND between the 
empirical measure of fully observed data and that of the imputed partially observed data. We assume that an interpolating minimizer exists, i.e.,  
$\Fc^*\cap\Fc_3\neq\emptyset.$
\end{assumption}

\begin{assumption}\label{assump:4}
Recall the empirical measure of mixture dataset 
\[
\pib_{ul}^f
=
\frac{n}{n+m}\pib_l^f
+
\frac{m}{n+m}\pib_u^f,
\]
and let
$\Fc_4
:= \{f: f\in \arginf_{f\in\Fc}
d_{\Gc_{nn}}\!(\pib_l,\pib_{ul}^f)\}.$ 
We assume that an interpolating mixture minimizer exists: $\Fc^*\cap\Fc_4\neq\emptyset.$
\end{assumption}

Assumption~\ref{assump:2} holds when $\Lc$ is point-identifying, $\Gc_{nn}$ is discriminative, and $\Fc$ can interpolate $\Db_l$. Concretely, it is shown in \cite{yun2019small} that a two-hidden-layer ReLU network can interpolate $N$
arbitrary input-output pairs with $d_z$-dimensional outputs when its widths $w_1,w_2$ satisfy 
\[
4\Big\lfloor\frac{w_1}{4}\Big\rfloor
\Big\lfloor\frac{w_2}{4d_z}\Big\rfloor\geq N.
\]
Taking $N=n$ supports Assumption~\ref{assump:2}, while $N=n+m$ allows the network to fit $\Db_l$ while preserving prescribed predictions on $\Db_u$, as required by
Assumption~\ref{assump:3}. Moreover, on a compact domain, the symmetric class of single-neuron ReLU discriminators is already discriminative \cite{zhang2017discrimination}. These conditions may fail with insufficient predictor capacity or conflicting repeated inputs. Assumption~\ref{assump:4} is an ideal mixture-compatibility condition that need not hold exactly in practice and will be relaxed in the subsequent {\color{black} non-ideal} analysis (Section \ref{sec:nonideal}).

\begin{theorem}\label{the:main}
Suppose Assumptions~\ref{assump:1}--\ref{assump:4} hold. 
Let $\hat{\mathfrak{R}}_n(\Gc_{nn})$ denote the empirical 
Rademacher complexity of $\Gc_{nn}$ on $n$ samples, and let
$\Hc := \{h(\xb)=g(\xb,f(\xb)): g\in\Gc_{nn}, f\in\Fc\}$. Then, 
with probability at least $1-6\delta$, 
the following bounds hold:
\begin{enumerate}
\item For a function $f_l^*\in\Fc$ learned with $n$ fully observed
samples without MCR through Eq.~\eqref{eq:objective}, we have
\begin{equation}
\label{eq:ideal_loss}
\begin{aligned}
R(f_l^*)-\inf_{f\in\Fc}R(f)
\leq{}&
4\hat{\mathfrak{R}}_n(\Gc_{nn})
+
2\hat{\mathfrak{R}}_n(\Hc)
+
\frac{3\Delta}{2\sqrt n},
\end{aligned}
\end{equation}
where
$\Delta:=3B_gB_{xz}\sqrt{2\ln\delta^{-1}}$.

\item For a function $f_u^*\in\Fc$ learned with $n$ fully observed
and $m$ partially observed samples with MCR through
Eq.~\eqref{eq:ndp}, we have
\begin{equation}
\label{eq:ideal_ndp}
\begin{aligned}
&R(f_u^*)-\inf_{f\in\Fc}R(f)
\leq{}
4\hat{\mathfrak{R}}_n(\Gc_{nn})
+
\frac{\Delta}{\sqrt n}\\
&+
\min\Bigg\{
2\hat{\mathfrak{R}}_n(\Hc)
+
\frac{\Delta}{2\sqrt n},
4\hat{\mathfrak{R}}_{m+n}(\Hc)
+
\frac{\Delta}{\sqrt{m+n}}
\Bigg\}.
\end{aligned}
\end{equation}

\end{enumerate}
\end{theorem}
}

{\color{black} Theorem~\ref{the:main} compares the high-probability estimation-error upper bounds of training with and without MCR under the ideal conditions in  Assumptions~\ref{assump:1}--\ref{assump:4}. The two bounds share the terms $4\hat{\mathfrak{R}}_n(\Gc_{nn})+\Delta/\sqrt n$.
Because an ideal MCR solution also interpolates the fully observed sample, the same $n$-sample term appearing in the vanilla bound always holds. Assumption~\ref{assump:4} provides the alternative $(m+n)$-sample bound. Consequently, the MCR upper bound is the shared term plus the minimum of these two alternatives and is therefore never larger than the vanilla upper bound under the stated ideal conditions.

The inequality is strict whenever
\[
4\hat{\mathfrak{R}}_{m+n}(\Hc)
+\frac{\Delta}{\sqrt{m+n}}
<
2\hat{\mathfrak{R}}_n(\Hc)
+\frac{\Delta}{2\sqrt n}.
\]
Therefore, increasing $m$ can tighten the finite-sample upper bound, although it neither guarantees a smaller realized error nor changes the asymptotic order in $n$. Under the usual scaling of the complexity as $\hat{\mathfrak{R}}_n(\Hc)=C_{\Hc}n^{-1/2}$, the MCR upper bound becomes smaller when $m>3n$. The reductions in both the complexity term and the associated $\Delta$ term are general statistical benefits shared with semi-supervised learning, consistent with Rademacher-complexity analyses showing that unlabeled samples can tighten estimation error bounds, including for vector-valued outputs \cite{oneto2015local,li2023semi}. For neural networks with bounded spectral norms, explicit architecture-dependent bounds follow from standard Rademacher-complexity estimates \cite{bartlett2017spectrally}.
}

\subsection{{\color{black} Non-ideal} Regime}\label{sec:nonideal}

{\color{black} The previous subsection considers an ideal regime in which the training objectives are optimized exactly, and the interpolation and compatibility conditions in Assumptions~\ref{assump:2}--\ref{assump:4} hold. We now consider the non-ideal regime, allowing non-zero supervised training error and an optimization gap in the MCR objective. Importantly, the analysis in this subsection no longer requires Assumptions~\ref{assump:2}--\ref{assump:4}. Our goal is to extend the previous theoretical results to this practical scenario to later provide algorithmic insights on handling the optimality gap (Section \ref{sec:handle}) and demonstrate the potential practical benefit of training with MCR.}

To this end, we first introduce a notion of \emph{loss function dominance} as follows.

\begin{definition}
Given the function class $\Gc$, a loss function $\Lc$ defined on the dataset $\Db={\{(\xb_i, \zb_i)\}}_{i=1}^n$ (with empirical distribution $\pib_l$), and a prediction function $f$, we say that $d_{\Gc}$ is $(C,\alpha)$-dominated by $\Lc$ if
\begin{equation*}
    d_{\Gc}(\pib_l, \pib_l^f) \leq C (\pib_l\Lc^f)^{\alpha},
\end{equation*}
for some constants $C,\alpha \in \R^+$, $\forall \Db \in (\Xc, \Zc)^n$ and $\forall f \in \Fc$.
\end{definition}

This definition is akin to popular bounding relationships between probability metrics \cite{gibbs2002choosing} and H\"older equivalence in metric studies \cite{prandi2014holder, mitchell1985carnot}, extended towards general loss functions. When $\Lc$ is a distance metric in the Euclidean space, this definition resembles classical smoothness conditions such as Lipschitz or Hölder continuity. As a prominent example, the distance generated by $\Gc_{nn}=\{\,g:\, g \text{ is $1$-Lipschitz }\}$ is $(1,1)$-dominated by the mean-absolute error and $(1, 0.5)$-dominated by the mean-squared error. 

This motivates considering the following condition under the non-ideal regime.   

\begin{assumption}\label{assump:5}
We assume that the primary loss function $\Lc$ in Eqs.~\eqref{eq:objective} and \eqref{eq:ndp} is non-negative with minimum value of $0$ for simplicity. The neural net class $\Gc_{nn}$ is chosen such that the induced NND is $(C,\alpha)$-dominated by $\Lc$, i.e., for any empirical distribution $\pib_l$ and prediction function $f$, we have
\begin{equation*}
    d_{\Gc_{nn}}(\pib_l, \pib_l^f) \leq C (\pib_l \Lc^f)^{\alpha},
\end{equation*}
for some constants $C,\alpha \in \R^+$. Furthermore, in the non-ideal regime, a learned function $f_l \in \Fc$ admits $\pib_l\Lc^{f_l} = \epsilon_{\Lc} \geq 0$.
\end{assumption}

In addition, the MCR framework introduces a potential optimization gap for the regularization term, motivating the following assumption.  

\begin{assumption}\label{assump:6}
For a predictor $f_u$ learned with MCR and the associated neural distance function 
\begin{equation*}
    g_u \in \underset{g \in \Gc_{nn}}{\arg\sup} \big| \pib_l g - \pib_u^{f_u} g \big|,
\end{equation*}
 we assume that $f_u$ achieves the same optimality gap on $\Lc$ as the baseline function $f_l$, i.e., 
\begin{equation*}
    \pib_l\Lc^{f_u} = \epsilon_{\Lc} \geq 0,
\end{equation*}
and that the penalty term admits an additional  gap defined as
\begin{equation}\label{eq:gap}
    \epsilon_d := d_{\Gc_{nn}}(\pib_l, \pib_u^{f_u}) - \underset{f \in \Fc}{\inf} \; d_{\Gc_{nn}}(\pib_l,\pib_u^f) \geq 0.
\end{equation}
\end{assumption}
In the above, $\epsilon_d$ can be thought of as the optimality gap in solving the minimization problem $d_{\Gc_{nn}}(\pib_l,\pib_u^f)$ over $f \in \Fc$  via a numerical solver. We can now establish the estimation error bound in the {\color{black} non-ideal} scenario.  

{\color{black} \begin{theorem}\label{the:imperfect}
Suppose Assumptions~\ref{assump:1}, \ref{assump:5},
and~\ref{assump:6} hold. Using the same notation as in
Theorem~\ref{the:main}, each of the following bounds holds
with probability at least $1-4\delta$: \begin{enumerate}
\item For $f_l\in\Fc$ trained on $n$ fully observed samples
without MCR,
\begin{equation}\label{eq:imperfect_loss}
\begin{aligned}
R(f_l)-\inf_{f\in\Fc}R(f)
\leq{}&
4\hat{\mathfrak{R}}_n(\Gc_{nn})
+
2\hat{\mathfrak{R}}_n(\Hc)
\\
&+
\frac{3\Delta}{2\sqrt n}
+
C(\epsilon_{\Lc})^\alpha ,
\end{aligned}
\end{equation}
where
$\Delta:=3B_gB_{xz}\sqrt{2\ln\delta^{-1}}$.

\item For $f_u\in\Fc$ trained on $n$ fully observed and $m$
partially observed samples with MCR,
\begin{equation}\label{eq:imperfect_ndp}
\begin{aligned}
R(f_u)-\inf_{f\in\Fc}R(f)
\leq &
4\hat{\mathfrak{R}}_n(\Gc_{nn})
+
2\hat{\mathfrak{R}}_{m+n}(\Hc)
\\
+&
\Delta\left(
\frac{1}{\sqrt n}
+
\frac{1}{2\sqrt{m+n}}
\right)
\\
+&
\frac{n}{m+n}C(\epsilon_{\Lc})^\alpha
+
\frac{m}{m+n}(\epsilon_d+\xi),
\end{aligned}
\end{equation}
where
\[
\xi:=
\inf_{f\in\Fc}
d_{\Gc_{nn}}(\pib_l,\pib_u^f).
\]

\end{enumerate}
\end{theorem}
}

{\color{black}
Relative to the ideal case, the non-zero training error contributes a loss-dependent term based on $\epsilon_{\Lc}$ to both bounds, while MCR additionally incurs a term dependent on $(\epsilon_d+\xi)$. Similar to Theorem~\ref{the:main}, the MCR upper bound in Theorem \ref{the:imperfect} could be presented as the minimum of two terms, suggesting that MCR training provides an upper bound that is no worse than the vanilla upper bound under practical conditions. However, the favorable upper bound does not always translate to an estimation error benefit as demonstrated later in Fig.~\ref{fig:dg_test}.The main decisive factor is the data incompatibility cost captured via $\xi$. This conditional behavior is consistent with empirical SSL studies showing that unlabeled data can degrade performance under distribution mismatch or unsuitable unlabeled samples \cite{oliver2018realistic,guo2020safe}. Nevertheless, the MCR bound motivates a  duality-gap-based diagnostic and a $\xi$-based prior viability check in Section~\ref{sec:handle}.}

\section{Handling the Optimality Gap}\label{sec:handle}

{\color{black} In this section, we develop a practical framework to manage the optimality gap $\epsilon_d$. From an upper-bound perspective, MCR offers an advantage over vanilla training only when $\epsilon_d$ is sufficiently small, such that the right-hand side (RHS) of Eq.~\eqref{eq:imperfect_loss} is larger than that of Eq.~\eqref{eq:imperfect_ndp}. In other words, a favorable upper-bound of MCR is achieved with high probability when
\begin{equation}\label{eq:robustness_condition}
\begin{aligned}
    \frac{m}{m+n} \big(\epsilon_d + \xi - C(\epsilon_{\Lc})^\alpha\big)
    \leq & 2\big(\hat{\mathfrak{R}}_{n}(\Hc) - \hat{\mathfrak{R}}_{m+n}(\Hc)\big) \\
    &+ \frac{\Delta}{2} \Big(\frac{1}{\sqrt{n}} - \frac{1}{\sqrt{m+n}}\Big).
\end{aligned}
\end{equation}

This raises a practical question: under what conditions can we maintain the upper-bound advantage? To answer this, we first characterize $\epsilon_d$ itself. 

Recall that $\epsilon_d$ is defined as the optimality gap of the penalty term in Eq.~\eqref{eq:gap}. To study it, we leverage the notion of \emph{duality gap} in min-max optimization. For a min-max problem on a generic function $\Lc$,
\begin{equation*}
    \underset{f \in \Fc}{\inf} \; \underset{g \in \Gc}{\sup} \; \Lc(f,g),
\end{equation*}
the \textit{duality gap} for a pair $(\hat{f},\hat{g})$ is defined as
\begin{equation*}
    DG(\hat{f},\hat{g}) := \underset{g \in \Gc}{\sup} \; \Lc(\hat{f},g) - \underset{f \in \Fc}{\inf} \; \Lc(f,\hat{g}).
\end{equation*}
Duality gap has been widely adopted as a practical convergence indicator for min-max problems, especially in generative adversarial training~\cite{grnarova2019domain, sidheekh2021duality, sidheekh2021characterizing}. Here, we employ it as a proxy to monitor and bound the optimality gap $\epsilon_d$.

\begin{theorem}\label{the:duality}
For $f_u \in \Fc$ and $g_u \in \Gc_{nn}$ as defined in Assumption \ref{assump:6}, the duality gap
\begin{equation*}
    DG(f_u,g_u) := \underset{g \in \Gc_{nn}}{\sup} \; |\pib_l g - \pib_u^{f_u} g| - \underset{f \in \Fc}{\inf} \; |\pib_l g_u - \pib_u^{f} g_u|,
\end{equation*}
upper bounds the optimality gap $\epsilon_d$ defined in Eq.~\eqref{eq:gap} as follows
\begin{equation*}
    \epsilon_d \leq DG(f_u,g_u).
\end{equation*}
\end{theorem}

\begin{proof}
Let $d\omega(f,g) = |\pib_l g - \pib_u^f g|$. We aim to show that
\begin{align*}
    d\omega(f_u,g_u) &- \underset{f \in \Fc}{\inf} \underset{g \in \Gc_{nn}}{\sup} \; d\omega(f,g) \\
    &\leq \underset{g \in \Gc_{nn}}{\sup} \; d\omega(f_u,g) - \underset{f \in \Fc}{\inf} \; d\omega(f,g_u).
\end{align*}
Observe that
\begin{equation*}
\begin{aligned}
    d\omega(f_u,g_u) &= \underset{g \in \Gc_{nn}}{\sup} \; d\omega(f_u,g), \\
    \underset{f \in \Fc}{\inf} \underset{g \in \Gc_{nn}}{\sup} \; d\omega(f,g) &\geq \underset{f \in \Fc}{\inf} \; d\omega(f,g_u),
\end{aligned}
\end{equation*}
where the first line follows from the definition of $g_u$ in Assumption \ref{assump:6}. Combining the two inequalities completes the proof.
\end{proof}

Theorem~\ref{the:duality} allows the optimization residual $\epsilon_d$ to be controlled through the duality gap, which can be monitored during training. Combining Theorem~\ref{the:duality} with Eq.~\eqref{eq:robustness_condition} directly gives the following sufficient condition.

\begin{corollary}\label{cor:condition}
For predictors $f_u$ and $f_l$ trained with and without MCR to the
same supervised loss precision $\epsilon_{\Lc}$, the MCR estimation-error
upper bound is no larger than the vanilla-training upper bound
whenever
\begin{equation}\label{eq:dg_bound_exact}
\begin{aligned}
DG(f_u,g_u)
\leq {}&
\frac{2(m+n)}{m}
\big(
\hat{\mathfrak{R}}_{n}(\Hc)
-
\hat{\mathfrak{R}}_{m+n}(\Hc)
\big)
\\
&+
\Delta
\frac{m+n-\sqrt{(m+n)n}}{2m\sqrt n}
+
C(\epsilon_{\Lc})^\alpha
-
\xi.
\end{aligned}
\end{equation}
\end{corollary}

Corollary~\ref{cor:condition} gives the exact sufficient condition implied by the upper-bound comparison. In particular, if the RHS of Eq.~\eqref{eq:dg_bound_exact} is negative, the sufficient condition cannot be satisfied because the duality gap is non-negative. However, several quantities in Eq.~\eqref{eq:dg_bound_exact} are difficult to estimate in practice. We therefore next consider practical simplifications and empirical estimates of these quantities.

The Rademacher complexity is generally difficult to estimate in practice. However, since $\hat{\mathfrak{R}}_{n}(\Hc)$ is usually larger than $\hat{\mathfrak{R}}_{m+n}(\Hc)$ when $m$ is large, 
a stronger sufficient condition can be obtained by dropping the Rademacher complexity terms from Eq.~\eqref{eq:robustness_condition}, which gives
\begin{equation}\label{eq:condition}
     \frac{m}{m+n} \big(\epsilon_d +  \xi - C(\epsilon_{\Lc})^\alpha\big) \leq \frac{\Delta}{2} \Big(\frac{1}{\sqrt{n}} - \frac{1}{\sqrt{m+n}}\Big).
\end{equation}
Such an ordering is generally expected as the sample size increases, particularly when $m$ is large. Under this simplification, satisfying Eq.~\eqref{eq:condition} is sufficient for satisfying the exact condition in Eq.~\eqref{eq:robustness_condition}.

On the other hand, the constant $\Delta$ is relatively easy to estimate: $B_{xz}$ can be approximated by the maximum pairwise distance in the data, while $B_g$ depends on the neural distance class $\Gc_{nn}$. This often yields a conservative estimate $\hat{\Delta}$, which further preserves the upper-bound advantage.

Exact evaluation of $\xi$ requires minimizing the joint
discrepancy over $\Fc$, which can be unstable for flexible predictors
and high-dimensional IPMs. Let us denote by $\pib_l(\xb)$ (respectively, $\pib_u(\xb)$) the marginal empirical measure of feature $\xb$ in the fully observed (respectively, partially observed) dataset. Instead of the joint
discrepancy, we use the marginal
discrepancy
\[
d_{\Gc_{nn}}\!\left(
\pib_l(\xb),\pib_u(\xb)
\right),
\]
which is independent of $\Fc$, as a tractable empirical proxy. Since
\[
d_{\Gc_{nn}}\!\left(
\pib_l(\xb),\pib_u(\xb)
\right)
\leq \xi,
\]
this proxy is not a formal upper bound on $\xi$ and is used only for practical diagnosis. We calibrate it against the intrinsic finite-sample IPM variability as follows:
\begin{enumerate}
    \item Sample multiple random subsets from $\Db_l$ with sizes consistent with MCR training batch sizes.
    \item Compute the average IPM estimate between the $\xb$ components of these subsets, denoted by $d_l$.
    \item Define the calibrated discrepancy as
    \begin{equation}\label{eq:xi}
        \hat{\xi}
        :=
        \max\Big\{
        d_{\Gc_{nn}}\big(\pib_l(\xb), \pib_u(\xb)\big)
        - d_l,\; 0
        \Big\}.
    \end{equation}
\end{enumerate}

The quantity $d_l$ estimates the IPM variability observed under random partitions of the fully observed data. Subtracting this baseline makes $\hat{\xi}$ a calibrated measure of excess discrepancy beyond ordinary finite-sample variation. It is a practical diagnostic rather than a guaranteed upper bound, unbiased estimator, or consistent estimator of $\xi$. The theoretical sufficient conditions above use the exact $\xi$, whereas the training protocol uses $\hat{\xi}$ as a plug-in approximation.

Let us now provide a simplified practical condition by putting everything together. Combining Theorem~\ref{the:duality} with the simplified condition in Eq.~\eqref{eq:condition} gives
\begin{equation}\label{eq:dg_bound_theory}
DG(f_u,g_u)
\leq
\Delta
\frac{m+n-\sqrt{(m+n)n}}{2m\sqrt n}
+
C(\epsilon_{\Lc})^\alpha
-
\xi.
\end{equation}
Eq.~\eqref{eq:dg_bound_theory} provides a simpler sufficient condition that does not require estimating the Rademacher complexity.

Using the empirical estimates $\hat{\Delta}$ and $\hat{\xi}$, we can first perform a prior check before training. In the well-trained regime where $\epsilon_{\Lc}\rightarrow 0$, the plug-in condition cannot be satisfied when
\begin{equation}\label{eq:viability}
\hat{\Delta}
\frac{m+n-\sqrt{(m+n)n}}{2m\sqrt n}
<
\hat{\xi},
\end{equation}
because the duality gap is non-negative. Eq.~\eqref{eq:viability} therefore provides a simple prior check on whether MCR can reach the favorable upper-bound regime under the plug-in approximation.

If the prior check is satisfied, we can further monitor the duality gap during training. Using the empirical estimates $\hat{\Delta}$ and $\hat{\xi}$, we have the following practical stopping diagnostic
\begin{equation}\label{eq:dg_bound}
DG(f_u,g_u)
\leq
\hat{\Delta}
\frac{m+n-\sqrt{(m+n)n}}{2m\sqrt n}
+
C(\epsilon_{\Lc})^\alpha
-
\hat{\xi}.
\end{equation}
The resulting condition is a bound-motivated training diagnostic rather than a formal certificate based on the unknown population quantities.

The accuracy of $\hat{\xi}$ controls the conservativeness of the plug-in diagnostic. In particular, if the supplied discrepancy is $\widetilde{\xi}=\xi+e_{\xi}$, the RHS of Eq.~\eqref{eq:dg_bound_theory} differs from its exact counterpart by $-e_{\xi}$. Underestimating $\xi$ therefore increases the RHS of the diagnostic equation, making it more permissive and potentially causing premature stopping, whereas overestimating $\xi$ decreases the RHS, making the diagnostic more conservative and potentially delaying or preventing stopping. The same directional effect applies to the prior check condition in Eq.~\eqref{eq:viability}. We formalize the MCR training procedure with the plug-in diagnostics in Algorithm~\ref{algo:wasserstein}.}

\begin{algorithm}
\caption{Batch MCR Training with Duality Gap Stopping Condition}

{\bf Require:} 
$f_{\theta}$, $g_{\gamma}$, $\Db_l$, $\Db_u$, an estimated $\hat{\Delta}$ and $\hat{\xi}$, loss $\Lc: \Xc \times \Zc \rightarrow \R$, Optimizer $O$, MCR update step $n_g$, Duality Gap update step $n_d$, MCR weight $\lambda_{d}$, $\theta$ learning rate $\alpha_{\theta}$ and convergence criterion, $\gamma$ learning rate $\alpha_{\gamma}$, $\tau = \text{True}$.

{\bf While:}
$\theta$ does not converge or $\tau$:
\begin{algorithmic}[1] 
\Procedure{MCR Training}{$\theta$}\label{algo:wasserstein}
    \State Divide $\Db_l$ into $k_l$ batches
    \For {$i \gets 1$ to $k_l$ }
        \For{$t \gets 1$ to $n_g$}
            \State Divide $\Db_u$ into $k_u$ batches
            \For{$j \gets 1$ to $k_u$}
                \State $d_j = \pib_u^{f_{\theta},j} g_{\gamma} - \pib_l^i g_{\gamma}$
            \EndFor
            \State $\gamma \gets O \big(\frac{1}{k_u}\sum d_j; \alpha_{\gamma} \big)$
        \EndFor
        \State $\theta \gets O \bigg(\pib_l^i \Lc - \lambda_d  \pib_u^{f_{\theta},k_u} g_{\gamma} ; \alpha_{\theta} \bigg)$ 
    \EndFor
    \If{$\theta$ converges}
        \State $\gamma' = \gamma$, $\theta' = \theta$, record $\epsilon_{\Lc}$
        \For{$q \gets 1$ to $n_d$}
            \State Get a random batch $\pib_l$ and $\pib_u$
            \State $d_q = \pib_u^{f_{\theta}} g_{\gamma'} - \pib_l g_{\gamma'}$
            \State $d'_q = \pib_u^{f_{\theta'}} g_{\gamma} - \pib_l g_{\gamma}$
            \State $\gamma' \gets O \big(d_q; \alpha_{\gamma} \big)$
            \State $\theta' \gets O \big(-d'_q; \alpha_{\theta} \big)$
        \EndFor
        \State Sample a large batch $\pib_l$ and $\pib_u$
        \State $d = \pib_u^{f_{\theta}} g_{\gamma'} - \pib_l g_{\gamma'}$
        \State $d' = \pib_u^{f_{\theta'}} g_{\gamma} - \pib_l g_{\gamma}$ 
        \If{$d'-d \leq$ RHS of Eq. \eqref{eq:dg_bound}}
            \State $\tau = \text{False}$
        \Else
            \State $\tau = \text{True}$
        \EndIf
    \EndIf
\EndProcedure
\end{algorithmic}
\end{algorithm}

\begin{remark}
We provide several clarifications regarding Algorithm \ref{algo:wasserstein}:
\begin{enumerate}
    \item Only in Algorithm \ref{algo:wasserstein}, we use $\pib_u^{f_{\theta},j} g_{\gamma}$ to denote expectation of $g_{\gamma}$ over the empirical distribution of the $j$-th batch of  partially observed dataset imputed with prediction function $f_{\theta}$. Note that during the $\theta$ update, the regularization is enforced with the $k_u$-th batch of the partially observed data, a common practice in adversarial training \cite{gulrajani2017improved}.
    \item In practice, extending the training time typically ensures the stopping condition is satisfied. {\color{black} Thus, the algorithm provides a bound-motivated convergence diagnostic for determining whether further optimization of the MCR objective is required under limited computational budgets.}
    \item In specific cases, additional constraints may need to be incorporated. For example, when employing the $1$-Wasserstein distance, it is necessary to enforce the $1$-Lipschitz constraint. This can be achieved by augmenting the MCR training process with a gradient penalty term \cite{gulrajani2017improved}.
\end{enumerate}
\end{remark}

\section{Simulations}
\label{sec:app}
\begin{table*}
    \centering
    \caption{Reconstruction Error Reduction Percentage (\%) with MCR. On columns title, we have the number of fully observed data $n$. On rows title, we have the ratio $m/n$. Statistically significant reductions (compared to zero as benchmark) are denoted in {\bf bold}.}
    \begin{tabular}{|c||c|c|c|c|c|}
    \hline
         \backslashbox{$m/n$}{$n$} & $10$ & $25$ & $50$ & $100$ & $200$ \\
         \hline\hline
        $2$ & $0.79 \pm 0.94$ & $-0.37 \pm 0.49$ & $\mathbf{2.25 \pm 0.61}$ & $\mathbf{2.97 \pm 0.61}$ & $\mathbf{2.72 \pm 0.4}$ \\
        \hline
        $10$ & $0.46 \pm 0.47$ & $0.24 \pm 0.49$ & $\mathbf{3.3 \pm 0.41}$ & $\mathbf{2.89 \pm 0.42}$ & $\mathbf{2.88 \pm 0.34}$ \\
        \hline
        $20$ & $\mathbf{0.76 \pm 0.69}$ & $\mathbf{2.1 \pm 0.44}$ & $\mathbf{3.6 \pm 0.38}$ & $\mathbf{4.26 \pm 0.5}$ & $\mathbf{3.71 \pm 0.36}$ \\
        \hline
    \end{tabular}
    \vspace{6pt}
    \label{tab:m&n}
\end{table*}
\begin{figure*}
    \centering
    \includegraphics[width=0.4\textwidth]{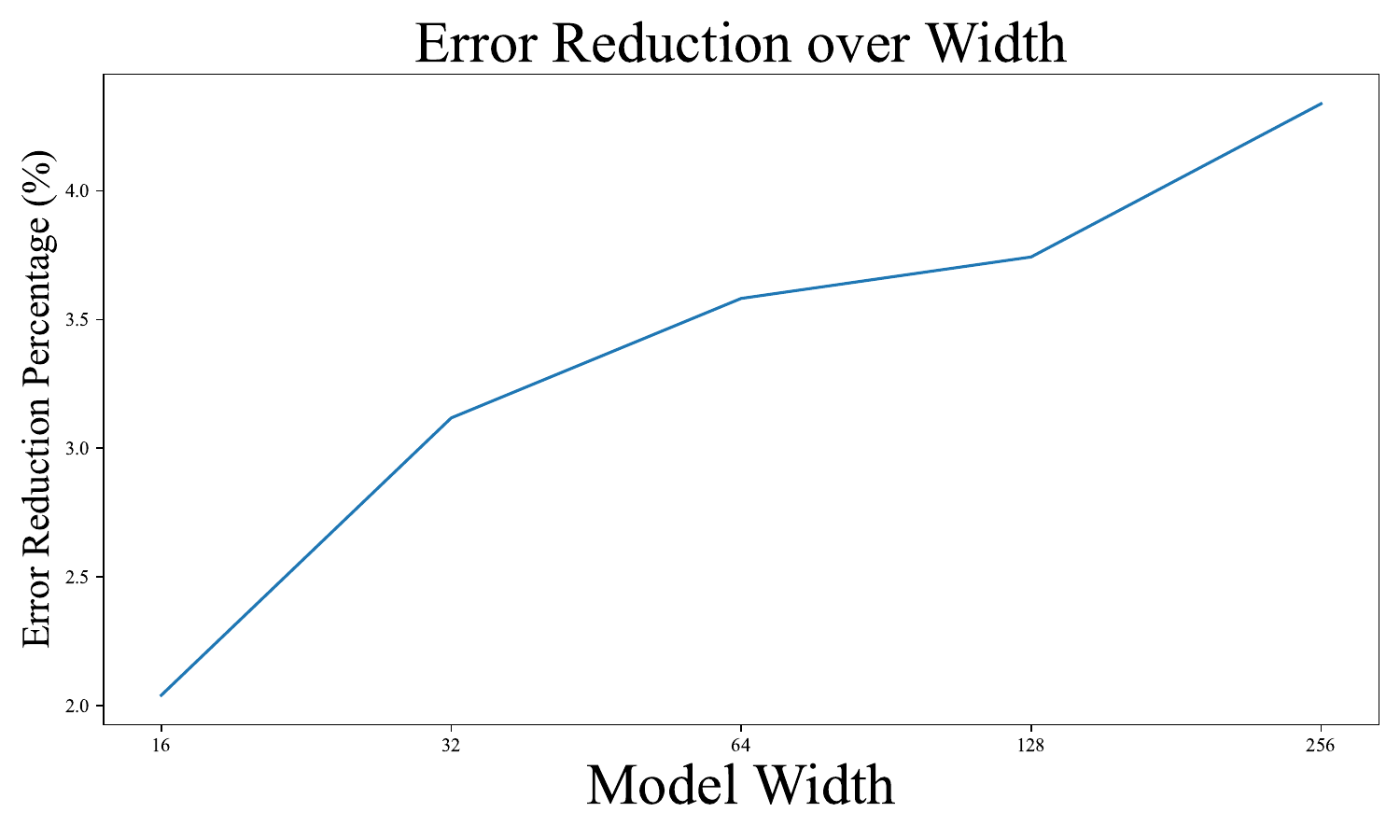}
    \includegraphics[width=0.4\textwidth]{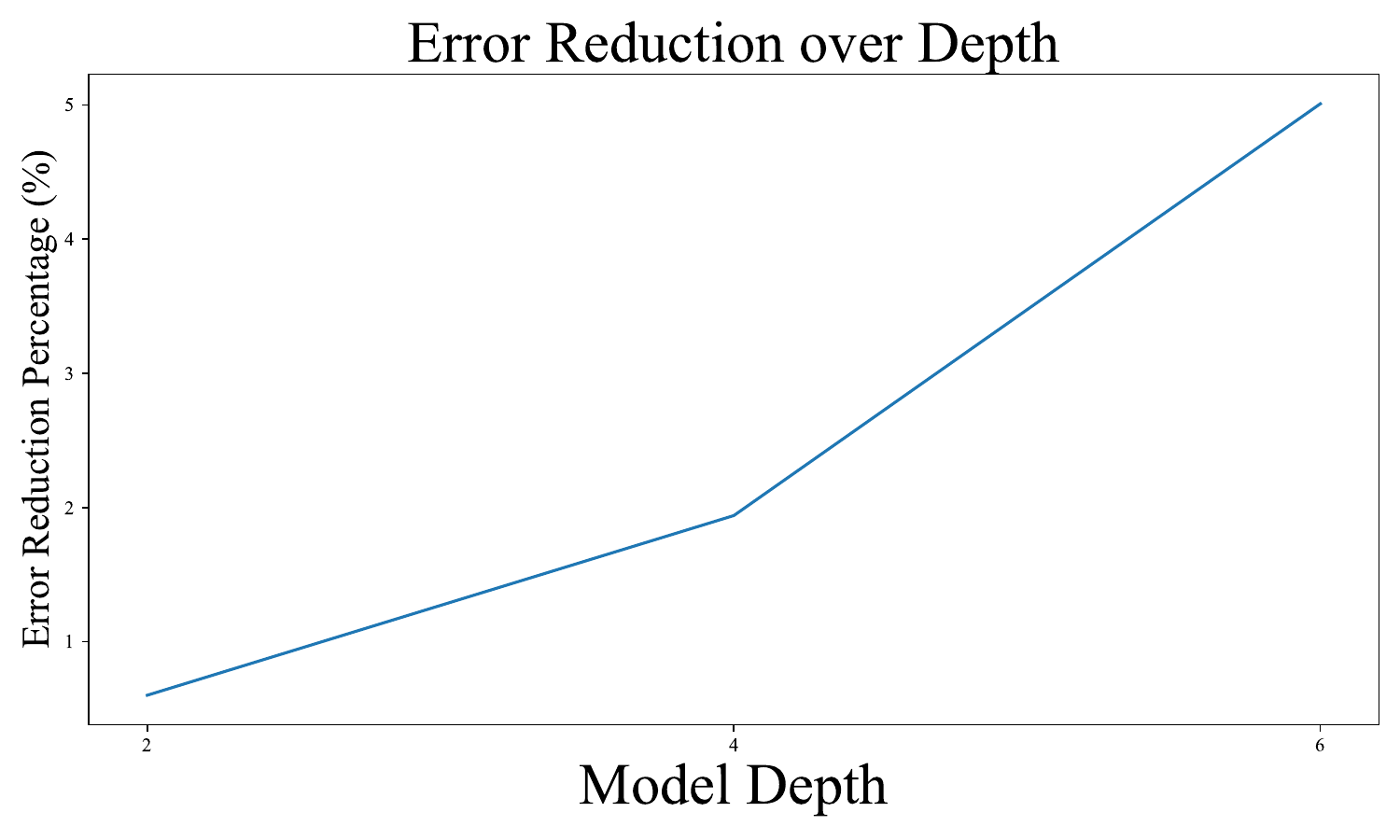}
    \caption{Reconstruction error reduction in percentage with MCR. {\bf Left:} The error reduction percentage over varying model width. {\bf Right:} The error reduction percentage over varying model depth.}
    \label{fig:capacity_error_syn}
\end{figure*}
We now present simulation results. For MCR training, we focus on the first-order Wasserstein distance as the regularization metric. This choice is motivated by its popularity among IPMs, its stability in gradient-based training \cite{arjovsky2017wasserstein}, and its optimization-friendly implementation with gradient penalties \cite{gulrajani2017improved}. In this case, the NND is defined with $\Gc$ as the class of $1-$Lipschitz functions, and we adopt the gradient penalty strategy of \cite{gulrajani2017improved}.

All implementation details---including model architecture, hyperparameter settings, data pre-processing, and optimizers—are provided in the supplementary materials for reproducibility. Each simulation is repeated for 20 Monte Carlo runs to compute error bars, unless otherwise stated.

\subsection{Verifying the Theoretical Claims}\label{subsec:theory}

We begin with ablation studies on a synthetic dataset to examine the theoretical claims. The target $\zb$ is generated by transforming $\xb$ through a randomly weighted multilayer perceptron (MLP), with $\xb \sim \mathcal{N}(0, I_{30})$ and $\zb \in \R^{20}$. We evaluate the reconstruction error reduction percentage:
\begin{equation*}
    \frac{\texttt{RMSE}(\zb,f^*_l(\xb)) - \texttt{RMSE}(\zb,f^*_u(\xb))}{\texttt{RMSE}(\zb,f^*_l(\xb))} \times 100\%,
\end{equation*}
where $f^*_l$ is learned by pure supervision (Eq.~\eqref{eq:objective}) and $f^*_u$ by MCR (Eq.~\eqref{eq:ndp}). {\color{black} In addition to being a popular loss metric, the root mean-squared error \texttt{RMSE} is linked to the $1$-Wasserstein distance used in our experiments. On the bounded domain considered here, the Lipschitz predictor $f$ makes $\Lc_f(\xb,\zb)=\|\zb-f(\xb)\|_2^2$ $K$-Lipschitz. Kantorovich--Rubinstein duality and the identity coupling
therefore give $W_1(\pib,\pib^f)\leq\texttt{RMSE}(\zb, f)
\leq\sqrt{K W_1(\pib,\pib^f)}$. Thus, \texttt{RMSE} and the Wasserstein discrepancy are equivalent up to the stated bounds.} Models in this subsection are selected using validation error.

\subsubsection*{Interplay between fully and partially observed data ($n$ and $m$)}
One favorable component in our theoretical bounds (Theorems~\ref{the:main} and \ref{the:imperfect}) is the difference
\[
\Delta\!\left(\frac{1}{\sqrt{n}} - \frac{1}{\sqrt{m+n}}\right),
\]
which captures how additional partially observed samples reduce the estimation error relative to pure supervision. This expression suggests two effects: (i) for fixed $n$, increasing $m$ should amplify the benefit of MCR; and (ii) for fixed $m/n$, larger $n$ should make this advantage more pronounced, as other error components diminish.

To test this prediction, we fix the data-generating distribution and use a two-hidden-layer MLP as the predictor class $\Fc$, while varying $n \in \{10,\ldots,200\}$ and the ratio $m/n \in \{2,10,20\}$. The resulting reconstruction error reductions are summarized in Table~\ref{tab:m&n}. When $n$ is small, a large $m$ is required to obtain statistically significant improvements, reflecting the dominance of the $1/\sqrt{n}$ estimation error term and the resulting high variance among empirically optimal predictors. As $n$ increases, MCR yields more stable and consistent gains, with improvements scaling predictably with $m/n$, qualitatively consistent with Theorem~\ref{the:main}.

\subsubsection*{Role of model capacity and approximation error} Note that our definition of estimation error can be written as the generalization error subtracted by the approximation error, with the latter given by $\inf_{f\in\Fc} R(f)$. While the approximation error is not directly observable, it is monotone with respect to model class inclusion: if $\Fc_1 \subseteq \Fc_2$, then $\inf_{f\in\Fc_1} R(f) \ge \inf_{f\in\Fc_2} R(f)$. Moreover, our theory suggests MCR's error bound advantage comes from estimation error, therefore, as approximation error decreases, MCR's benefit should become more apprant.

We examine this effect by systematically increasing model capacity, varying both the width (from $16$ to $256$ hidden units) and depth (from $2$ to $6$ layers) of the MLP predictor, while keeping $n$ and $m$ fixed. Fig.~\ref{fig:capacity_error_syn} shows that the reconstruction error reduction percentage increases monotonically with both width and depth. This suggests that MCR could be more beneficial for richer function classes.

\begin{figure}
    \centering
    \includegraphics[width=\linewidth]{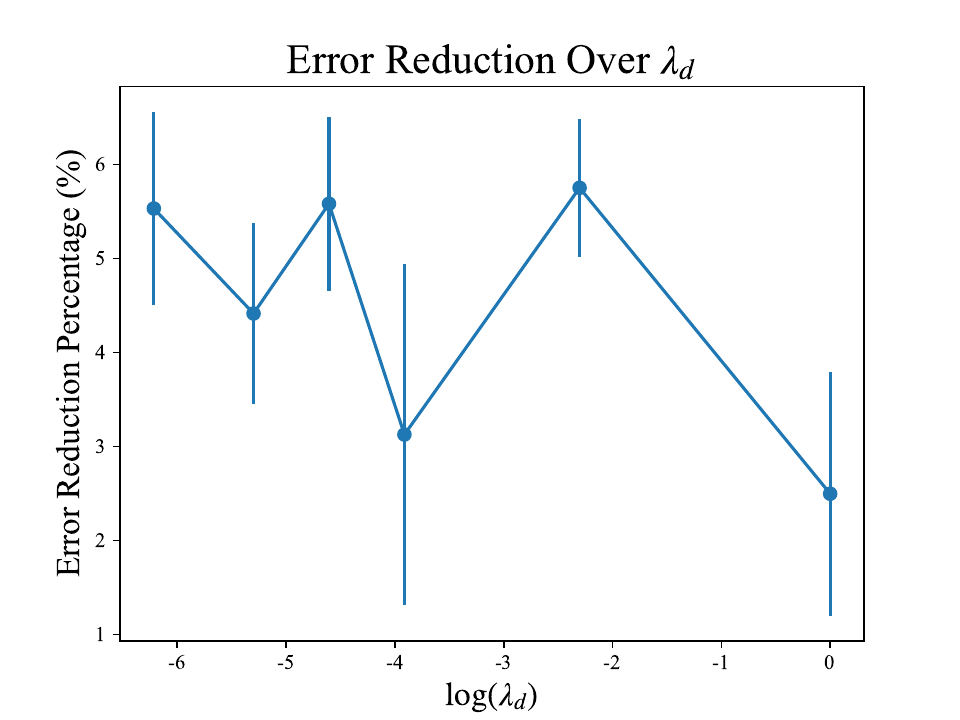}
    \caption{Error reduction vs $\log \lambda_d$. We do not observe a statistically significant trend across different $\lambda_d$.}
    \label{fig:lambdad}
\end{figure}

\subsubsection*{Effect of the MCR weight $\lambda_d$}
Our theoretical guarantees are conditioned on achieving a non-zero but finite training loss tolerance $\epsilon_{\Lc}$, and do not depend explicitly on the choice of $\lambda_d$. Consequently, we expect comparable performance gains across a broad range of $\lambda_d$ values, provided that training reaches the prescribed loss threshold.

To verify this, we fix $(n,m)$ at the best-performing configuration from Table~\ref{tab:m&n} and use a two-layer MLP with width $64$ as predictor. We vary $\lambda_d \in \{0.002, 0.005, 0.01, 0.02, 0.1, 1\}$, training each model until $\epsilon_{\Lc} \le 0.003$ and the duality gap stopping condition is satisfied. As shown in Fig.~\ref{fig:lambdad}, there is no statistically significant trend across different $\lambda_d$ values, confirming that $\lambda_d$ primarily affects optimization dynamics rather than the theoretical benefit of MCR. In practice, we therefore recommend choosing $\lambda_d$ small enough to avoid impeding convergence of the primary loss $\Lc$. 

\begin{table*}
    \centering
    \caption{Reconstruction Error Reduction Percentage (\%) with MCR. On rows title, we have the $m/n$ ratio. On columns title, we have the model widths. In each cell, we show MMD performance on the \textbf{left}, and $1-$Wasserstein distance performance on the \textbf{right}. The higher value in each cell is marked in \textbf{bold}.}
    \begin{tabular}{|c||c|c|c|}
    \hline
         \backslashbox{$m/n$}{Model Width} & $64$ & $128$ & $256$ \\
         \hline\hline
        $2$ & $\mathbf{2.05 \pm 0.54}$/$1.86 \pm 0.61$ & $2.66 \pm 0.54$/$\mathbf{2.98 \pm 0.36}$ & $2.84 \pm 0.60$/$\mathbf{3.08 \pm 0.54}$   \\
        \hline
        $10$ & $0.24 \pm 0.41$/$\mathbf{0.76 \pm 0.44}$ & $3.40 \pm 0.50$/$\mathbf{3.69 \pm 1.11}$ &  $3.14 \pm 0.37$/$\mathbf{3.42 \pm 0.46}$ \\
        \hline
        $20$ & $0.79 \pm 0.77$/$\mathbf{1.71 \pm 0.98}$ & $2.99 \pm 0.44$/$\mathbf{3.81 \pm 0.18}$ & $4.15 \pm 0.35$/$\mathbf{4.98 \pm 0.82}$  \\
        \hline
    \end{tabular}
    \vspace{6pt}
    \label{tab:mmdvw}
\end{table*}

\subsubsection*{Choice of neural net distance $\Gc_{nn}$}
While our {\color{black} non-ideal} regime analysis relies on loss dominance conditions, the underlying intuition of MCR, constraining the imputed distribution to remain close to the fully observed one, should hold across a range of discrepancy measures. To assess this robustness, we compare MCR using Gaussian-kernel MMD (bandwidth $1$) and $1$-Wasserstein distance.

We vary the ratio $m/n \in \{2,10,20\}$ and the model width in $\{64,128,256\}$, while fixing the depth at $4$ and using $n=100$. The results, reported in Table~\ref{tab:mmdvw}, show that Wasserstein-based MCR typically achieves slightly larger error reductions, though MMD-based MCR also yields consistent and statistically significant improvements in most settings. This aligns with the fact that MMD can be $(C,\alpha)$-dominated by \texttt{RMSE} under suitable RKHS smoothness assumptions.

\subsubsection*{Distribution discrepancy and prior-check sensitivity}

{\color{black} We next examine how distribution discrepancy affects the empirical
benefit of MCR and evaluate the sensitivity of the prior viability
check in Eq.~\eqref{eq:viability}. We inject mean shifts of magnitude
\[
\{0,\;0.25,\;0.5,\;0.75,\;1,\;1.5,\;2,\;3,\;4\},
\]
into the partially observed sample while fixing the sample sizes and
model architecture. Each shift level is evaluated over five Monte
Carlo runs, giving 45 cases in total.

Because the exact $\xi$ is intractable, the practical check uses the
calibrated estimate $\hat{\xi}$ in Eq.~\eqref{eq:xi}. As shown in
Fig.~\ref{fig:xi}, the \texttt{RMSE} reduction generally decreases as
$\hat{\xi}$ increases, and the nominal plug-in boundary lies near the
transition between the beneficial and the non-beneficial MCR trainings.} 

\begin{figure}
    \centering
    \includegraphics[width=\linewidth]{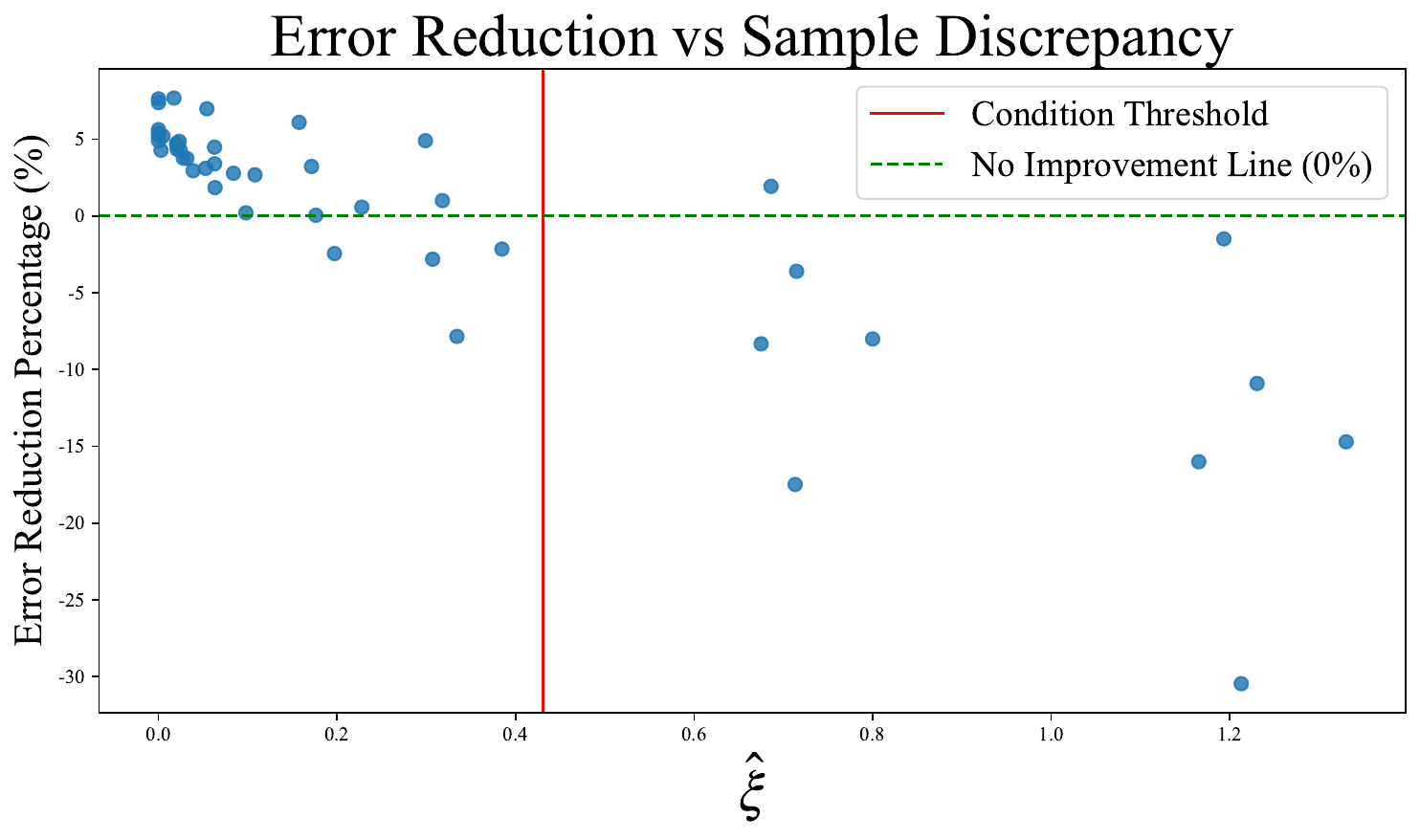}
    \caption{Error reduction vs $\hat{\xi}$. The condition threshold in LHS of Eq.~\eqref{eq:viability} based on our $\hat{\xi}$ estimate in Eq.~\eqref{eq:xi} lies near the empirical transition between beneficial and non-beneficial MCR.}
    \label{fig:xi}
\end{figure}

{\color{black} We further perturb the calibrated discrepancy as
$\hat{\xi}_{\rho}=\rho\hat{\xi}$, where $\rho<1$ represents
underestimation and $\rho>1$ represents overestimation. A case is
empirically positive when MCR achieves lower test \texttt{RMSE} than vanilla
training, and predictively positive when the prior check permits MCR.
Table~\ref{tab:xi_sensitivity} reports the resulting decision
sensitivity.

\begin{table}[h]
    \centering
    \caption{Sensitivity of the prior MCR viability check (Eq.~\eqref{eq:viability}) to multiplicative perturbations of $\hat{\xi}$. FPR and FNR denote false-positive and false-negative rates, respectively.}
    \label{tab:xi_sensitivity}
    \footnotesize
    \setlength{\tabcolsep}{3.2pt}
    \begin{tabular}{c|cc|ccc}
        \toprule
        $\rho$ & FP & FN & FPR (\%) & FNR (\%) & Accuracy (\%) \\
        \midrule
        $0.20$ & $13$ & $0$ & $100.0$ & $0.0$  & $71.1$ \\
        $0.50$ & $8$  & $0$ & $61.5$  & $0.0$  & $82.2$ \\
        $0.75$ & $4$  & $1$ & $30.8$  & $3.1$  & $88.9$ \\
        $1.00$ & $4$  & $1$ & $30.8$  & $3.1$  & $88.9$ \\
        $1.25$ & $3$  & $1$ & $23.1$  & $3.1$  & $91.1$ \\
        $2.00$ & $1$  & $4$ & $7.7$   & $12.5$ & $88.9$ \\
        $5.00$ & $0$  & $9$ & $0.0$   & $28.1$ & $80.0$ \\
        \bottomrule
    \end{tabular}
\end{table}

At the nominal calibration $\rho=1$, the prior check correctly classifies $40$ of the $45$ cases, including $31$ true positives, $9$ true negatives, $4$ false positives and $1$ false negative. Best accuracies are achieved for $\rho\in[0.75,2]$. As expected,
underestimating $\hat{\xi}$ makes the check more permissive and primarily increases false positives, whereas overestimating $\hat{\xi}$ makes it more conservative and primarily increases false
negatives.}

\subsection{Evaluating the Duality Gap Stopping Condition}
\begin{figure*}
    \centering
    \includegraphics[width=0.4\textwidth]{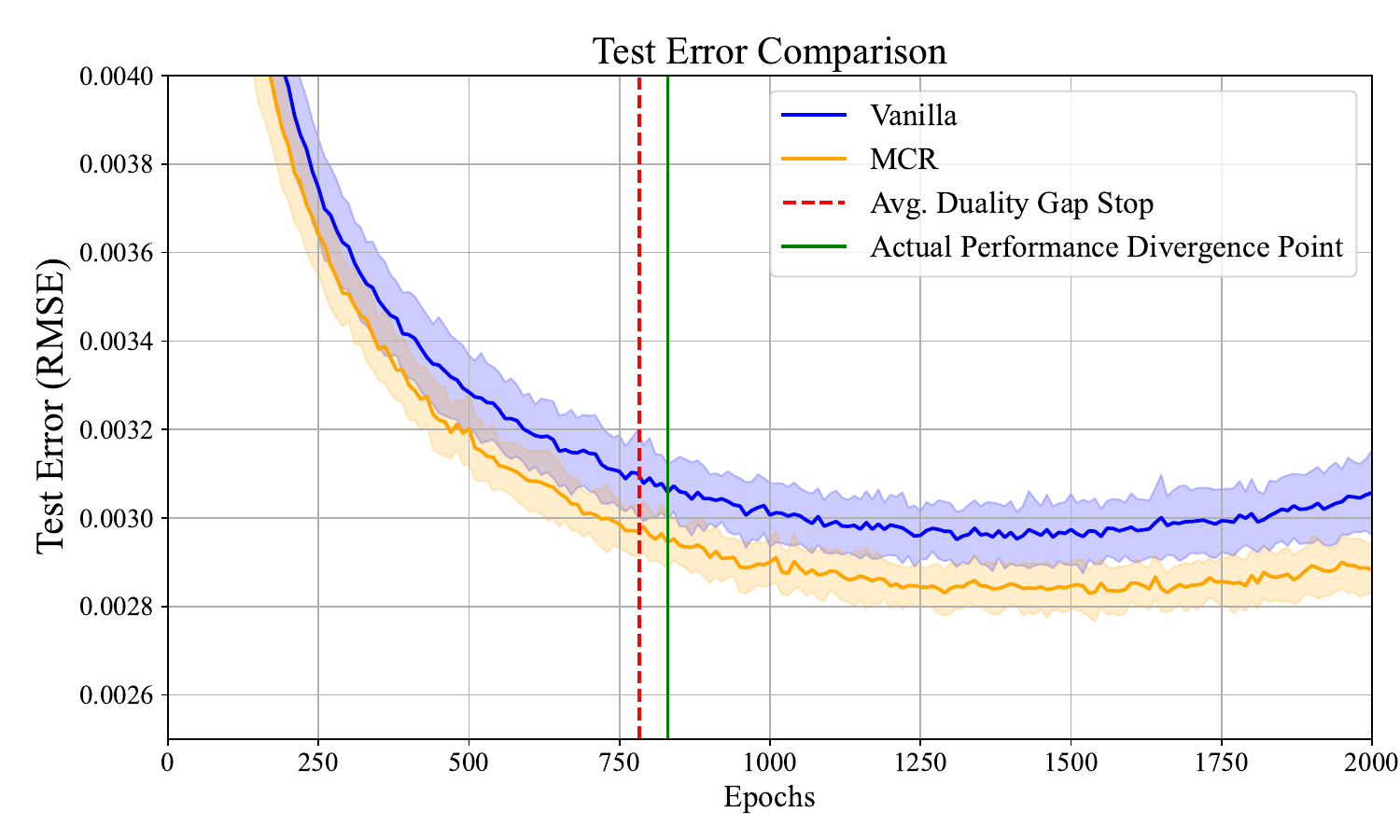}
    \includegraphics[width=0.4\textwidth]{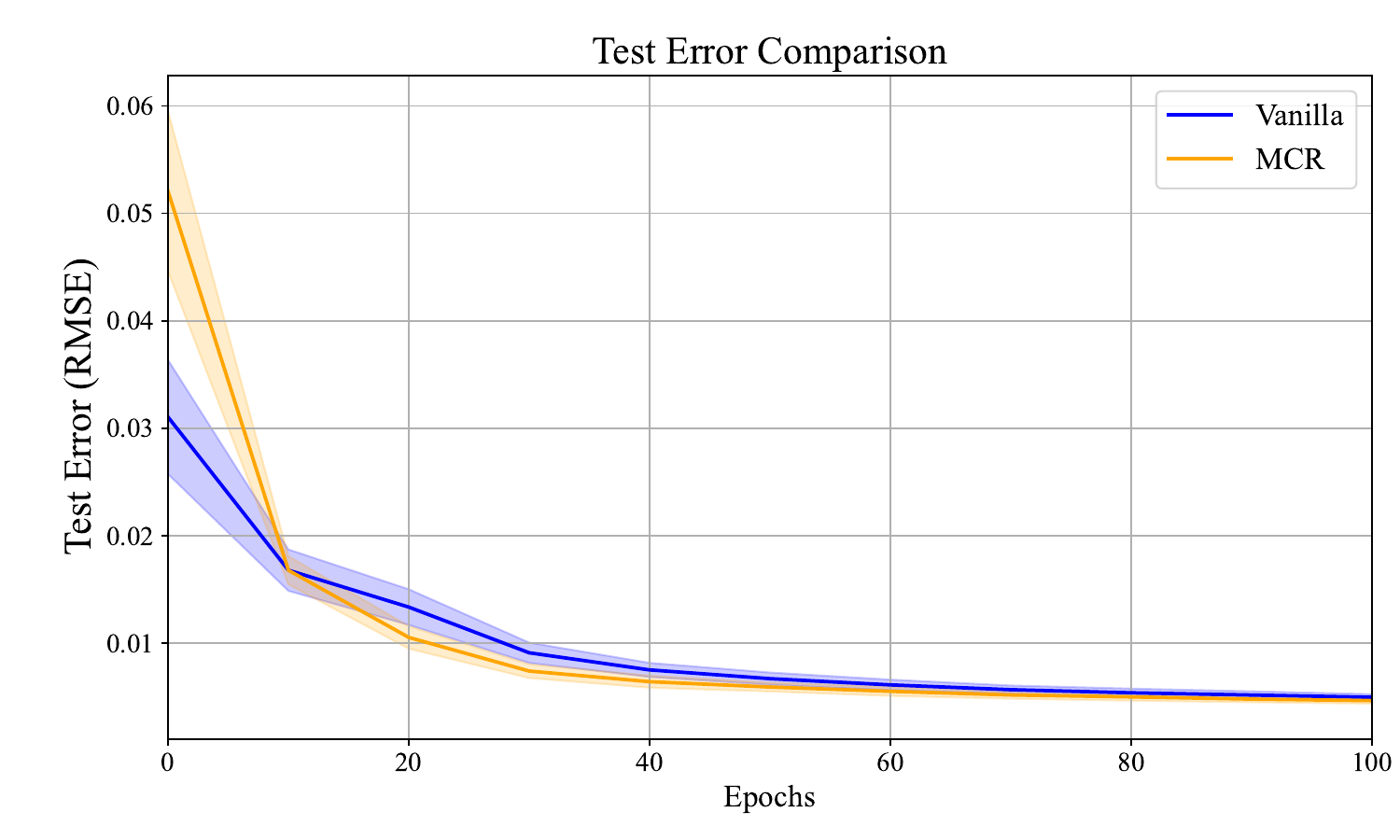}
    \caption{The test error curve comparison between vanilla training and MCR training. \textbf{Left:} Test error curve over $2000$ epochs. \textbf{Right:} Test error curve comparison in the first $100$ epochs.}
    \label{fig:dg_test}
\end{figure*}

We now examine the effectiveness of the duality gap stopping condition. Importantly, this criterion does not alter the MCR training process itself; rather, it determines the appropriate stopping point in Algorithm~\ref{algo:wasserstein}. Experimental settings are detailed in the supplementary materials. Using the same synthetic data and model configuration as in Section~\ref{subsec:theory}, we train both vanilla and MCR models for $2000$ epochs. To ensure comparability, we apply the same optimizer and learning rate across both settings. Convergence of $\theta$ in Algorithm~\ref{algo:wasserstein} is defined by the training loss $\Lc$ (not the neural net distance) stabilizing within $0.00005$ over five consecutive epochs.

Fig.~\ref{fig:dg_test} reports the test data \texttt{RMSE} comparisons alongside the stopping points selected by the duality gap condition. On the left, we show the full $2000$-epoch curves, overlaying two key markers: (i) the average duality gap stopping epoch, and (ii) the performance divergence epoch, defined as the first point where the test error confidence intervals of vanilla and MCR training no longer overlap. The right panel zooms in on the first 100 epochs. Several observations emerge:
\begin{enumerate}
    \item MCR training achieves consistently lower test error compared with vanilla training, {\color{black} consistent with the favorable-bound regime.}
    \item The average duality gap stopping epoch is close to the performance divergence epoch, {\color{black} which illustrates the practical timing of the diagnostic in this setting.}
    \item In the initial training phase, MCR temporarily yields higher test error compared with vanilla training. This phenomenon reflects the theoretical insight from Theorem~\ref{the:imperfect}, that performance gains are not always true at all stages, and it empirically confirms the existence of “bad” MCR training phases.
\end{enumerate}

\subsection{Versatility of MCR}

{\color{black} We evaluate MCR on five structured-missingness tasks spanning three data modalities. For MNIST and CIFAR-10, the upper half of each image is used to reconstruct its lower half with a convolutional encoder--decoder. For the Activity and Crop datasets, variational autoencoders impute ECG signals from TEB and EDA measurements \cite{mohino2019activity}, and optical measurements from radar data \cite{khosravi2018msmd,khosravi2019random}, respectively. Finally, we use MultiVI \cite{ashuach2023multivi} to impute ATAC-seq measurements
from RNA counts in the scMultiSim dataset \cite{li2022scmultisim}. Performance is evaluated by \texttt{RMSE}.

To compare representative measure-matching objectives under a common experimental protocol, we instantiate Eq.~\eqref{eq:MCR} using a $W_1$ neural critic, a logistic GAN discrepancy, a Gaussian-kernel MMD, a Sinkhorn divergence, and an exact quadratic Wasserstein distance. We term the 1-Wasserstein based MCR implementation as \textbf{MCR-W} to highlight its penalty choice. The GAN objective is the structured-missingness counterpart of GAIN \cite{yoon2018gain}: because the entire $\zb$ block is absent in $\Db_u$, the entry-wise mask and hint mechanisms are unnecessary. The Gaussian-kernel MMD variant is then equivalent to GAIN with an MMD GAN. We term these two variants \textbf{GAIN} and \textbf{GAIN-MMD}, respectively. The Sinkhorn and $W_2$ variants similarly specialize the OT formulation of \cite{muzellec2020missing} to the available complete reference measure $\pib_l$. We term these two \textbf{OT-Sinkhorn} and \textbf{OT-$\mathbf{W_2}$}. All methods use the same task-specific prediction architecture and data splits; only the distributional penalty and its associated optimization parameters differ. Complete settings are reported in the supplementary material.

\begin{table*}[h]
    \centering
    \caption{Comparison of imputation error across datasets. Results are reported as
    mean $\pm$ standard error. Lower error is better. The last row reports the
    average percentage improvement relative to Vanilla across all datasets.}
    \label{tab:benchmark_results}
    \small
    \setlength{\tabcolsep}{3.5pt}
    \begin{tabular}{lcccccc}
        \toprule
        Dataset
        & Vanilla
        & MCR-W
        & GAIN
        & GAIN-MMD
        & OT-Sinkhorn
        & OT-$W_2$ \\
        \midrule
        MNIST
        & $0.0447 \pm 0.0002$
        & $0.0446 \pm 0.0002$
        & $0.0446 \pm 0.0002$
        & $0.0445 \pm 0.0002$
        & $0.0440 \pm 0.0002$
        & $0.0440 \pm 0.0002$ \\

        CIFAR
        & $0.0487 \pm 0.0002$
        & $0.0461 \pm 0.0002$
        & $0.0534 \pm 0.0005$
        & $0.0461 \pm 0.0002$
        & $0.0461 \pm 0.0003$
        & $0.0461 \pm 0.0003$ \\

        Activity
        & $0.7863 \pm 0.0391$
        & $0.7523 \pm 0.0268$
        & $0.7501 \pm 0.0311$
        & $0.7584 \pm 0.0251$
        & $0.7155 \pm 0.0287$
        & $0.7134 \pm 0.0278$ \\

        Crop
        & $0.5545 \pm 0.0087$
        & $0.5538 \pm 0.0099$
        & $0.5633 \pm 0.0094$
        & $0.5500 \pm 0.0092$
        & $0.5443 \pm 0.0091$
        & $0.5438 \pm 0.0091$ \\

        scMultiSim
        & $0.3742 \pm 0.0056$
        & $0.3582 \pm 0.0056$
        & $0.4483 \pm 0.0073$
        & $0.3596 \pm 0.0068$
        & $0.3542 \pm 0.0069$
        & $0.3503 \pm 0.0064$ \\
        \midrule
        Avg. $\Delta$ vs. Vanilla (\%) $\uparrow$
        & --
        & $+2.87$
        & $-5.20$
        & $+2.79$
        & $+4.62$
        & $+4.91$ \\
        \bottomrule
    \end{tabular}
\end{table*}

Table~\ref{tab:benchmark_results} shows that MCR-W, GAIN-MMD, OT-Sinkhorn, and OT-$W_2$ achieve lower mean \texttt{RMSE} than Vanilla on all five tasks, with average reductions of $2.87\%$, $2.79\%$, $4.62\%$, and $4.91\%$, respectively. In contrast, GAIN exhibits mixed performance, improving on MNIST and Activity but degrading on the remaining tasks. These results demonstrate that the benefit is not tied to a single discrepancy while also showing that the choice
of measure-matching penalty matters in practice.

\begin{table*}[h]
    \centering
    \caption{Accumulated penalty computation time (seconds) across datasets. Results are reported as mean $\pm$ standard error over 10 runs. Lower values represent higher computational efficiency.}
    \label{tab:penalty_runtime}
    \small
    \setlength{\tabcolsep}{4.0pt}
    \begin{tabular}{lccccc}
        \toprule
        Dataset
        & MCR-W
        & GAIN
        & GAIN-MMD
        & OT-Sinkhorn
        & OT-$W_2$ \\
        \midrule
        MNIST
        & $3.73 \pm 0.15$
        & $8.36 \pm 0.36$
        & $2.51 \pm 0.05$
        & $37.34 \pm 4.29$
        & $35.76 \pm 0.24$ \\

        CIFAR
        & $4.49 \pm 0.02$
        & $11.13 \pm 0.06$
        & $2.94 \pm 0.01$
        & $48.37 \pm 0.12$
        & $75.46 \pm 0.04$ \\

        Activity
        & $9.14 \pm 0.05$
        & $15.76 \pm 0.07$
        & $7.61 \pm 0.13$
        & $97.08 \pm 0.52$
        & $27.79 \pm 0.24$ \\

        Crop
        & $5.53 \pm 0.04$
        & $9.76 \pm 0.30$
        & $4.29 \pm 0.05$
        & $70.78 \pm 0.19$
        & $21.11 \pm 0.34$ \\

        scMultiSim
        & $3.37 \pm 0.11$
        & $6.19 \pm 0.13$
        & $2.48 \pm 0.05$
        & $30.22 \pm 0.43$
        & $60.48 \pm 0.18$ \\
        \bottomrule
    \end{tabular}
\end{table*}

The runtime results in Table~\ref{tab:penalty_runtime} reflect the expected computational tradeoff. Suppressing dimension and network evaluation constants, neural-critic and GAN penalties scale linearly with batch size per adversarial update, empirical MMD requires $\mathcal{O}(B^2)$ kernel evaluations, $T$ Sinkhorn iterations require
$\mathcal{O}(TB^2)$ operations, and exact equal-weight $W_2$ matching requires $\mathcal{O}(B^3)$ assignment time. At the batch sizes used here, direct MMD is the fastest, while $W_1$-MCR requires $3.37$--$9.14$ seconds of accumulated penalty computation and is substantially less expensive than both OT variants. A duality-gap
evaluation additionally requires $n_d$ predictor/critic update steps and therefore has cost $\mathcal{O}\!\left(n_d B(\mathcal{C}_f+\mathcal{C}_g)\right)$, where $\mathcal{C}_f$ and $\mathcal{C}_g$ denote the per-sample
evaluation costs of the predictor and critic.
}

\section{Conclusion}
In this work, we presented a unified framework for measure consistency regularization (MCR) in imputation tasks involving partially observed datasets. Our analysis focused on the theoretical underpinnings of MCR through the lens of neural net distance, leading to estimation error bounds that address both the ideal and {\color{black} non-ideal} regimes. This analysis addressed three critical questions concerning the benefits of MCR. (i) Under {\color{black} ideal optimality and compatibility conditions}, we showed that the primary advantage stems from a favorable estimation-error upper bound. {\color{black} (ii) In the non-ideal regime, we showed that the potential advantage depends on the optimization error and the compatibility between the fully and partially observed data. (iii) We further derived an exact duality-gap-based sufficient condition for the high-probability upper-bound comparison and introduced a calibrated plug-in diagnostic for practical training.}

We validated these theoretical findings through controlled simulations of the predicted behavior of the bounds and the proposed training protocol. {\color{black} On five structured-imputation tasks, we further compared state-of-the-art imputation methods under a common MCR framework. The results showed that several measure-matching objectives
improve upon vanilla training, while their accuracy and computational costs differ substantially across penalties and tasks.} 

Finally, we believe that our work has broader implications beyond imputation. The imputation framework considered here is closely related to post-training methods widely adopted in large language models (LLMs). As such, our theoretical analysis could be extended to two-stage training frameworks, offering new insights into the mechanisms and effectiveness of post-training in modern machine learning systems.

\ifCLASSOPTIONcompsoc
  \section*{Acknowledgments}
\else
  \section*{Acknowledgment}
\fi

This work was initiated when Yinsong Wang was at Northeastern University and his work was supported by the NSF Award \#2038625 as part of the NSF/DHS/DOT/NIH/USDA-NIFA Cyber-Physical Systems Program.

{\color{black} We would like to thank the anonymous reviewers and the editor for their insightful comments and suggestions to help improve the quality of this work.}

\ifCLASSOPTIONcaptionsoff
  \newpage
\fi

\bibliographystyle{IEEEtran}
\bibliography {bibliography.bib}

@article{ji2021understanding,
  title={Understanding estimation and generalization error of generative adversarial networks},
  author={Ji, Kaiyi and Zhou, Yi and Liang, Yingbin},
  journal={IEEE Transactions on Information Theory},
  volume={67},
  number={5},
  pages={3114--3129},
  year={2021},
  publisher={IEEE}
}

@InProceedings{pmlr-v70-arora17a,
  title = 	 {Generalization and Equilibrium in Generative Adversarial Nets ({GAN}s)},
  author =       {Sanjeev Arora and Rong Ge and Yingyu Liang and Tengyu Ma and Yi Zhang},
  booktitle = 	 {Proceedings of the 34th International Conference on Machine Learning},
  pages = 	 {224--232},
  year = 	 {2017},
  editor = 	 {Precup, Doina and Teh, Yee Whye},
  volume = 	 {70},
  series = 	 {Proceedings of Machine Learning Research},
  month = 	 {06--11 Aug},
  publisher =    {PMLR},
  url = 	 {https://proceedings.mlr.press/v70/arora17a.html}
}

@inproceedings{joy2021learning,
  title={Learning Multimodal VAEs through Mutual Supervision},
  author={Joy, Tom and Shi, Yuge and Torr, Philip and Rainforth, Tom and Schmon, Sebastian M and Siddharth, N},
  booktitle={International Conference on Learning Representations},
  year={2021}
}

@inproceedings{muzellec2020missing,
  title={Missing data imputation using optimal transport},
  author={Muzellec, Boris and Josse, Julie and Boyer, Claire and Cuturi, Marco},
  booktitle={International Conference on Machine Learning},
  pages={7130--7140},
  year={2020},
  organization={PMLR}
}

@inproceedings{arjovsky2017wasserstein,
  title={Wasserstein generative adversarial networks},
  author={Arjovsky, Martin and Chintala, Soumith and Bottou, L{\'e}on},
  booktitle={International conference on machine learning},
  pages={214--223},
  year={2017},
  organization={PMLR}
}

@article{gulrajani2017improved,
  title={Improved training of wasserstein gans},
  author={Gulrajani, Ishaan and Ahmed, Faruk and Arjovsky, Martin and Dumoulin, Vincent and Courville, Aaron C},
  journal={Advances in neural information processing systems},
  volume={30},
  year={2017}
}

@article{ji2018minimax,
  title={Minimax estimation of neural net distance},
  author={Ji, Kaiyi and Liang, Yingbin},
  journal={Advances in Neural Information Processing Systems},
  volume={31},
  year={2018}
}

@article{liang2017well,
  title={How well can generative adversarial networks learn densities: A nonparametric view},
  author={Liang, Tengyuan},
  journal={arXiv preprint arXiv:1712.08244},
  year={2017}
}

@article{zhang2017discrimination,
  title={On the discrimination-generalization tradeoff in GANs},
  author={Zhang, Pengchuan and Liu, Qiang and Zhou, Dengyong and Xu, Tao and He, Xiaodong},
  journal={arXiv preprint arXiv:1711.02771},
  year={2017}
}

@inproceedings{yoon2018gain,
  title={Gain: Missing data imputation using generative adversarial nets},
  author={Yoon, Jinsung and Jordon, James and Schaar, Mihaela},
  booktitle={International conference on machine learning},
  pages={5689--5698},
  year={2018},
  organization={PMLR}
}

@article{long2018conditional,
  title={Conditional adversarial domain adaptation},
  author={Long, Mingsheng and Cao, Zhangjie and Wang, Jianmin and Jordan, Michael I},
  journal={Advances in neural information processing systems},
  volume={31},
  year={2018}
}

@article{li2020maximum,
  title={Maximum density divergence for domain adaptation},
  author={Li, Jingjing and Chen, Erpeng and Ding, Zhengming and Zhu, Lei and Lu, Ke and Shen, Heng Tao},
  journal={IEEE transactions on pattern analysis and machine intelligence},
  volume={43},
  number={11},
  pages={3918--3930},
  year={2020},
  publisher={IEEE}
}

@book{rachev2013methods,
  title={The methods of distances in the theory of probability and statistics},
  author={Rachev, Svetlozar T and Klebanov, Lev B and Stoyanov, Stoyan V and Fabozzi, Frank},
  volume={10},
  year={2013},
  publisher={Springer}
}

@article{kantorovich1958space,
  title={On a space of totally additive functions},
  author={Kantorovich, Leonid Vasilevich and Rubinshtein, SG},
  journal={Vestnik of the St. Petersburg University: Mathematics},
  volume={13},
  number={7},
  pages={52--59},
  year={1958},
  publisher={Allerton Press, Inc.}
}

@article{ASENS_1953_3_70_3_267_0,
     author = {Fortet, R. and Mourier, E.},
     title = {Convergence de la r\'epartition empirique vers la r\'epartition th\'eorique},
     journal = {Annales scientifiques de l'\'Ecole Normale Sup\'erieure},
     pages = {267--285},
     publisher = {Elsevier},
     volume = {3e s{\'e}rie, 70},
     number = {3},
     year = {1953},
     doi = {10.24033/asens.1013},
     mrnumber = {15,808d},
     zbl = {0053.09601},
     language = {fr},
     url = {http://www.numdam.org/articles/10.24033/asens.1013/}
}

@inproceedings{yoon2020gamin,
  title={GAMIN: Generative adversarial multiple imputation network for highly missing data},
  author={Yoon, Seongwook and Sull, Sanghoon},
  booktitle={Proceedings of the IEEE/CVF conference on computer vision and pattern recognition},
  pages={8456--8464},
  year={2020}
}

@inproceedings{wu2023jointly,
  title={Jointly imputing multi-view data with optimal transport},
  author={Wu, Yangyang and Miao, Xiaoye and Huang, Xinyu and Yin, Jianwei},
  booktitle={Proceedings of the AAAI Conference on Artificial Intelligence},
  volume={37},
  number={4},
  pages={4747--4755},
  year={2023}
}

@inproceedings{tran2017missing,
  title={Missing modalities imputation via cascaded residual autoencoder},
  author={Tran, Luan and Liu, Xiaoming and Zhou, Jiayu and Jin, Rong},
  booktitle={Proceedings of the IEEE conference on computer vision and pattern recognition},
  pages={1405--1414},
  year={2017}
}

@inproceedings{ma2021smil,
  title={Smil: Multimodal learning with severely missing modality},
  author={Ma, Mengmeng and Ren, Jian and Zhao, Long and Tulyakov, Sergey and Wu, Cathy and Peng, Xi},
  booktitle={Proceedings of the AAAI Conference on Artificial Intelligence},
  volume={35},
  number={3},
  pages={2302--2310},
  year={2021}
}

@inproceedings{zhao2023transformed,
  title={Transformed distribution matching for missing value imputation},
  author={Zhao, He and Sun, Ke and Dezfouli, Amir and Bonilla, Edwin V},
  booktitle={International Conference on Machine Learning},
  pages={42159--42186},
  year={2023},
  organization={PMLR}
}

@inproceedings{sun2024redcore,
  title={RedCore: Relative Advantage Aware Cross-modal Representation Learning for Missing Modalities with Imbalanced Missing Rates},
  author={Sun, Jun and Zhang, Xinxin and Han, Shoukang and Ruan, Yu-Ping and Li, Taihao},
  booktitle={Proceedings of the AAAI Conference on Artificial Intelligence},
  volume={38},
  number={13},
  pages={15173--15182},
  year={2024}
}

@article{cheng2024collaboratively,
  title={Collaboratively Learning Linear Models with Structured Missing Data},
  author={Cheng, Chen and Cheng, Gary and Duchi, John C},
  journal={Advances in Neural Information Processing Systems},
  volume={36},
  year={2024}
}

@article{kang2023cm,
  title={CM-GAN: A cross-modal generative adversarial network for imputing completely missing data in digital industry},
  author={Kang, Mingyu and Zhu, Ran and Chen, Duxin and Liu, Xiaolu and Yu, Wenwu},
  journal={IEEE Transactions on Neural Networks and Learning Systems},
  year={2023},
  publisher={IEEE}
}

@article{ashuach2023multivi,
  title={MultiVI: deep generative model for the integration of multimodal data},
  author={Ashuach, Tal and Gabitto, Mariano I and Koodli, Rohan V and Saldi, Giuseppe-Antonio and Jordan, Michael I and Yosef, Nir},
  journal={Nature Methods},
  volume={20},
  number={8},
  pages={1222--1231},
  year={2023},
  publisher={Nature Publishing Group US New York}
}

@article{tu2022cross,
  title={Cross-linked unified embedding for cross-modality representation learning},
  author={Tu, Xinming and Cao, Zhi-Jie and Mostafavi, Sara and Gao, Ge and others},
  journal={Advances in Neural Information Processing Systems},
  volume={35},
  pages={15942--15955},
  year={2022}
}

@article{heumos2023best,
  title={Best practices for single-cell analysis across modalities},
  author={Heumos, Lukas and Schaar, Anna C and Lance, Christopher and Litinetskaya, Anastasia and Drost, Felix and Zappia, Luke and L{\"u}cken, Malte D and Strobl, Daniel C and Henao, Juan and Curion, Fabiola and others},
  journal={Nature Reviews Genetics},
  volume={24},
  number={8},
  pages={550--572},
  year={2023},
  publisher={Nature Publishing Group UK London}
}

@article{gong2021cobolt,
  title={Cobolt: integrative analysis of multimodal single-cell sequencing data},
  author={Gong, Boying and Zhou, Yun and Purdom, Elizabeth},
  journal={Genome biology},
  volume={22},
  pages={1--21},
  year={2021},
  publisher={Springer}
}

@article{gayoso2021joint,
  title={Joint probabilistic modeling of single-cell multi-omic data with totalVI},
  author={Gayoso, Adam and Steier, Zo{\"e} and Lopez, Romain and Regier, Jeffrey and Nazor, Kristopher L and Streets, Aaron and Yosef, Nir},
  journal={Nature methods},
  volume={18},
  number={3},
  pages={272--282},
  year={2021},
  publisher={Nature Publishing Group US New York}
}

@article{li2022scmultisim,
  title={scMultiSim: simulation of multi-modality single cell data guided by cell-cell interactions and gene regulatory networks},
  author={Li, Hechen and Zhang, Ziqi and Squires, Michael and Chen, Xi and Zhang, Xiuwei},
  journal={Research Square},
  year={2022},
  publisher={American Journal Experts}
}

@article{cohen2023joint,
  title={Joint variational autoencoders for multimodal imputation and embedding},
  author={Cohen Kalafut, Noah and Huang, Xiang and Wang, Daifeng},
  journal={Nature Machine Intelligence},
  volume={5},
  number={6},
  pages={631--642},
  year={2023},
  publisher={Nature Publishing Group UK London}
}

@article{mitra2023learning,
  title={Learning from data with structured missingness},
  author={Mitra, Robin and McGough, Sarah F and Chakraborti, Tapabrata and Holmes, Chris and Copping, Ryan and Hagenbuch, Niels and Biedermann, Stefanie and Noonan, Jack and Lehmann, Brieuc and Shenvi, Aditi and others},
  journal={Nature Machine Intelligence},
  volume={5},
  number={1},
  pages={13--23},
  year={2023},
  publisher={Nature Publishing Group UK London}
}

@article{bahador2021reconstruction,
  title={Reconstruction of missing channel in electroencephalogram using spatiotemporal correlation-based averaging},
  author={Bahador, Nooshin and Jokelainen, Jarno and Mustola, Seppo and Kortelainen, Jukka},
  journal={Journal of Neural Engineering},
  volume={18},
  number={5},
  pages={056045},
  year={2021},
  publisher={IOP Publishing}
}

@article{hou2020systematic,
  title={A systematic evaluation of single-cell RNA-sequencing imputation methods},
  author={Hou, Wenpin and Ji, Zhicheng and Ji, Hongkai and Hicks, Stephanie C},
  journal={Genome biology},
  volume={21},
  pages={1--30},
  year={2020},
  publisher={Springer}
}

@article{sohn2020fixmatch,
  title={Fixmatch: Simplifying semi-supervised learning with consistency and confidence},
  author={Sohn, Kihyuk and Berthelot, David and Carlini, Nicholas and Zhang, Zizhao and Zhang, Han and Raffel, Colin A and Cubuk, Ekin Dogus and Kurakin, Alexey and Li, Chun-Liang},
  journal={Advances in neural information processing systems},
  volume={33},
  pages={596--608},
  year={2020}
}

@article{berthelot2019mixmatch,
  title={Mixmatch: A holistic approach to semi-supervised learning},
  author={Berthelot, David and Carlini, Nicholas and Goodfellow, Ian and Papernot, Nicolas and Oliver, Avital and Raffel, Colin A},
  journal={Advances in neural information processing systems},
  volume={32},
  year={2019}
}

@article{huang2024flatmatch,
  title={FlatMatch: Bridging Labeled Data and Unlabeled Data with Cross-Sharpness for Semi-Supervised Learning},
  author={Huang, Zhuo and Shen, Li and Yu, Jun and Han, Bo and Liu, Tongliang},
  journal={Advances in Neural Information Processing Systems},
  volume={36},
  year={2024}
}

@inproceedings{taherkhani2020transporting,
  title={Transporting labels via hierarchical optimal transport for semi-supervised learning},
  author={Taherkhani, Fariborz and Dabouei, Ali and Soleymani, Sobhan and Dawson, Jeremy and Nasrabadi, Nasser M},
  booktitle={Computer Vision--ECCV 2020: 16th European Conference, Glasgow, UK, August 23--28, 2020, Proceedings, Part IV 16},
  pages={509--526},
  year={2020},
  organization={Springer}
}

@inproceedings{taherkhani2021self,
  title={Self-supervised wasserstein pseudo-labeling for semi-supervised image classification},
  author={Taherkhani, Fariborz and Dabouei, Ali and Soleymani, Sobhan and Dawson, Jeremy and Nasrabadi, Nasser M},
  booktitle={Proceedings of the IEEE/CVF conference on computer vision and pattern recognition},
  pages={12267--12277},
  year={2021}
}

@article{mohino2019activity,
  title={Activity recognition using wearable physiological measurements: Selection of features from a comprehensive literature study},
  author={Mohino-Herranz, Inma and Gil-Pita, Roberto and Rosa-Zurera, Manuel and Seoane, Fernando},
  journal={Sensors},
  volume={19},
  number={24},
  year={2019},
  publisher={MDPI}
}

@article{khosravi2019random,
  title={A random forest-based framework for crop mapping using temporal, spectral, textural and polarimetric observations},
  author={Khosravi, Iman and Alavipanah, Seyed Kazem},
  journal={International Journal of Remote Sensing},
  volume={40},
  number={18},
  pages={7221--7251},
  year={2019},
  publisher={Taylor \& Francis}
}

@article{khosravi2018msmd,
  title={MSMD: maximum separability and minimum dependency feature selection for cropland classification from optical and radar data},
  author={Khosravi, Iman and Safari, Abdolreza and Homayouni, Saeid},
  journal={International Journal of Remote Sensing},
  volume={39},
  number={8},
  pages={2159--2176},
  year={2018},
  publisher={Taylor \& Francis}
}

@article{akidau2015dataflow,
  title={The dataflow model: a practical approach to balancing correctness, latency, and cost in massive-scale, unbounded, out-of-order data processing},
  author={Akidau, Tyler and Bradshaw, Robert and Chambers, Craig and Chernyak, Slava and Fern{\'a}ndez-Moctezuma, Rafael J and Lax, Reuven and McVeety, Sam and Mills, Daniel and Perry, Frances and Schmidt, Eric and others},
  journal={Proceedings of the VLDB Endowment},
  volume={8},
  number={12},
  pages={1792--1803},
  year={2015},
  publisher={VLDB Endowment}
}

@article{fazlyab2019efficient,
  title={Efficient and accurate estimation of lipschitz constants for deep neural networks},
  author={Fazlyab, Mahyar and Robey, Alexander and Hassani, Hamed and Morari, Manfred and Pappas, George},
  journal={Advances in neural information processing systems},
  volume={32},
  year={2019}
}

@article{grnarova2019domain,
  title={A domain agnostic measure for monitoring and evaluating GANs},
  author={Grnarova, Paulina and Levy, Kfir Y and Lucchi, Aurelien and Perraudin, Nathanael and Goodfellow, Ian and Hofmann, Thomas and Krause, Andreas},
  journal={Advances in neural information processing systems},
  volume={32},
  year={2019}
}

@inproceedings{sidheekh2021duality,
  title={On duality gap as a measure for monitoring gan training},
  author={Sidheekh, Sahil and Aimen, Aroof and Madan, Vineet and Krishnan, Narayanan C},
  booktitle={2021 international joint conference on neural networks (IJCNN)},
  pages={1--8},
  year={2021},
  organization={IEEE}
}

@inproceedings{sidheekh2021characterizing,
  title={On characterizing GAN convergence through proximal duality gap},
  author={Sidheekh, Sahil and Aimen, Aroof and Krishnan, Narayanan C},
  booktitle={International Conference on Machine Learning},
  pages={9660--9670},
  year={2021},
  organization={PMLR}
}

@phdthesis{choudhury2020missing,
  title={Missing data imputation using machine learning and natural language processing for clinical diagnostic codes},
  author={Choudhury, Arkopal},
  year={2020},
  school={The University of North Carolina at Chapel Hill}
}

@article{swanson2021simultaneous,
  title={Simultaneous trimodal single-cell measurement of transcripts, epitopes, and chromatin accessibility using {TEA}-seq},
  author={Swanson, Elliott and Lord, Cara and Reading, Julian and Heubeck, Alexander T and Genge, Palak C and Thomson, Zachary and Weiss, Morgan DA and Li, Xiao-jun and Savage, Adam K and Green, Richard R and Torgerson, Troy R and Bumol, Thomas F and Graybuck, Lucas T and Skene, Peter J},
  journal={Elife},
  volume={10},
  pages={e63632},
  year={2021},
  publisher={eLife Sciences Publications, Ltd}
}

@article{mimitou2021scalable,
  title={Scalable, multimodal profiling of chromatin accessibility, gene expression and protein levels in single cells},
  author={Mimitou, Eleni P and Lareau, Caleb A and Chen, Kelvin Y and Zorzetto-Fernandes, Andre L and Hao, Yuhan and Takeshima, Yusuke and Luo, Wendy and Huang, Tse-Shun and Yeung, Bertrand Z and Papalexi, Efthymia and Thakore, Pratiksha I and Kibayashi, Tatsuya and Wing, James B and Hata, Mayu and Satija, Rahul and Nazor, Kristopher I and Sakaguchi, Shimon and Ludwig, Leif S and Sankaran, Vijay G and Regev, Aviv and Smibert, Peter},
  journal={Nature Biotechnology},
  volume={39},
  number={10},
  pages={1246--1258},
  year={2021},
  publisher={Nature Publishing Group US New York}
}

@article{garcia2009k,
  title={K nearest neighbours with mutual information for simultaneous classification and missing data imputation},
  author={Garc{\'\i}a-Laencina, Pedro J and Sancho-G{\'o}mez, Jos{\'e}-Luis and Figueiras-Vidal, An{\'\i}bal R and Verleysen, Michel},
  journal={Neurocomputing},
  volume={72},
  number={7-9},
  pages={1483--1493},
  year={2009},
  publisher={Elsevier}
}

@article{yu2020regression,
  title={Regression multiple imputation for missing data analysis},
  author={Yu, Lili and Liu, Liang and Peace, Karl E},
  journal={Statistical Methods in Medical Research},
  volume={29},
  number={9},
  pages={2647--2664},
  year={2020},
  publisher={SAGE Publications Sage UK: London, England}
}

@article{zhao2016multiple,
  title={Multiple imputation in the presence of high-dimensional data},
  author={Zhao, Yize and Long, Qi},
  journal={Statistical Methods in Medical Research},
  volume={25},
  number={5},
  pages={2021--2035},
  year={2016},
  publisher={SAGE Publications Sage UK: London, England}
}

@article{elharrouss2020image,
  title={Image inpainting: {A} review},
  author={Elharrouss, Omar and Almaadeed, Noor and Al-Maadeed, Somaya and Akbari, Younes},
  journal={Neural Processing Letters},
  volume={51},
  number={2},
  pages={2007--2028},
  year={2020},
  publisher={Springer}
}

@article{smieja2018processing,
  title={Processing of missing data by neural networks},
  author={{\'S}mieja, Marek and Struski, {\L}ukasz and Tabor, Jacek and Zieli{\'n}ski, Bartosz and Spurek, Przemys{\l}aw},
  journal={The 2018 Advances in Neural Information Processing Systems},
  volume={31},
  month={Dec},
  year={2018},
  location={Montréal, Canada}
}

@article{van2020survey,
  title={A survey on semi-supervised learning},
  author={Van Engelen, Jesper E and Hoos, Holger H},
  journal={Machine Learning},
  volume={109},
  number={2},
  pages={373--440},
  year={2020},
  publisher={Springer}
}

@book{villani2008optimal,
  title={Optimal Transport: Old and New},
  author={Villani, C{\'e}dric and others},
  volume={338},
  year={2008},
  publisher={Springer}
}

@article{gibbs2002choosing,
  title={On choosing and bounding probability metrics},
  author={Gibbs, Alison L and Su, Francis Edward},
  journal={International Statistical Review},
  volume={70},
  number={3},
  pages={419--435},
  year={2002},
  publisher={Wiley Online Library}
}

@article{mitchell1985carnot,
  title={On carnot-carath{\'e}odory metrics},
  author={Mitchell, John},
  journal={Journal of Differential Geometry},
  volume={21},
  number={1},
  pages={35--45},
  year={1985},
  publisher={Lehigh University}
}

@article{prandi2014holder,
  title={H{\"o}lder equivalence of the value function for control-affine systems},
  author={Prandi, Dario},
  journal={ESAIM: Control, Optimisation and Calculus of Variations},
  volume={20},
  number={4},
  pages={1224--1248},
  year={2014}
}

@inproceedings{wang2019semi,
  title={Semi-supervised learning by augmented distribution alignment},
  author={Wang, Qin and Li, Wen and Gool, Luc Van},
  booktitle={Proceedings of the IEEE/CVF international conference on computer vision},
  pages={1466--1475},
  year={2019}
}

@article{bartlett2017spectrally,
  title={Spectrally-normalized margin bounds for neural networks},
  author={Bartlett, Peter L and Foster, Dylan J and Telgarsky, Matus J},
  journal={Advances in neural information processing systems},
  volume={30},
  year={2017}
}

@article{zhang2012generalization,
  title={Generalization bounds for domain adaptation},
  author={Zhang, Chao and Zhang, Lei and Ye, Jieping},
  journal={Advances in neural information processing systems},
  volume={25},
  year={2012}
}

@article{yun2019small,
  title={Small relu networks are powerful memorizers: a tight analysis of memorization capacity},
  author={Yun, Chulhee and Sra, Suvrit and Jadbabaie, Ali},
  journal={Advances in neural information processing systems},
  volume={32},
  year={2019}
}

@article{oliver2018realistic,
  title={Realistic evaluation of deep semi-supervised learning algorithms},
  author={Oliver, Avital and Odena, Augustus and Raffel, Colin and Cubuk, Ekin Dogus and Goodfellow, Ian},
  journal={Advances in neural information processing systems},
  volume={31},
  year={2018}
}

@inproceedings{guo2020safe,
  title={Safe deep semi-supervised learning for unseen-class unlabeled data},
  author={Guo, Lan-Zhe and Zhang, Zhen-Yu and Jiang, Yuan and Li, Yu-Feng and Zhou, Zhi-Hua},
  booktitle={International conference on machine learning},
  pages={3897--3906},
  year={2020},
  organization={PMLR}
}

@article{oneto2015local,
  title={Local rademacher complexity: Sharper risk bounds with and without unlabeled samples},
  author={Oneto, Luca and Ghio, Alessandro and Ridella, Sandro and Anguita, Davide},
  journal={Neural Networks},
  volume={65},
  pages={115--125},
  year={2015},
  publisher={Elsevier}
}

@article{najafi2019robustness,
  title={Robustness to adversarial perturbations in learning from incomplete data},
  author={Najafi, Amir and Maeda, Shin-ichi and Koyama, Masanori and Miyato, Takeru},
  journal={Advances in Neural Information Processing Systems},
  volume={32},
  year={2019}
}

@article{li2023semi,
  title={Semi-supervised vector-valued learning: Improved bounds and algorithms},
  author={Li, Jian and Liu, Yong and Wang, Weiping},
  journal={Pattern Recognition},
  volume={138},
  pages={109356},
  year={2023},
  publisher={Elsevier}
}

\begin{IEEEbiography}
	[{\includegraphics[width=1in,height=1.25in,clip,keepaspectratio]{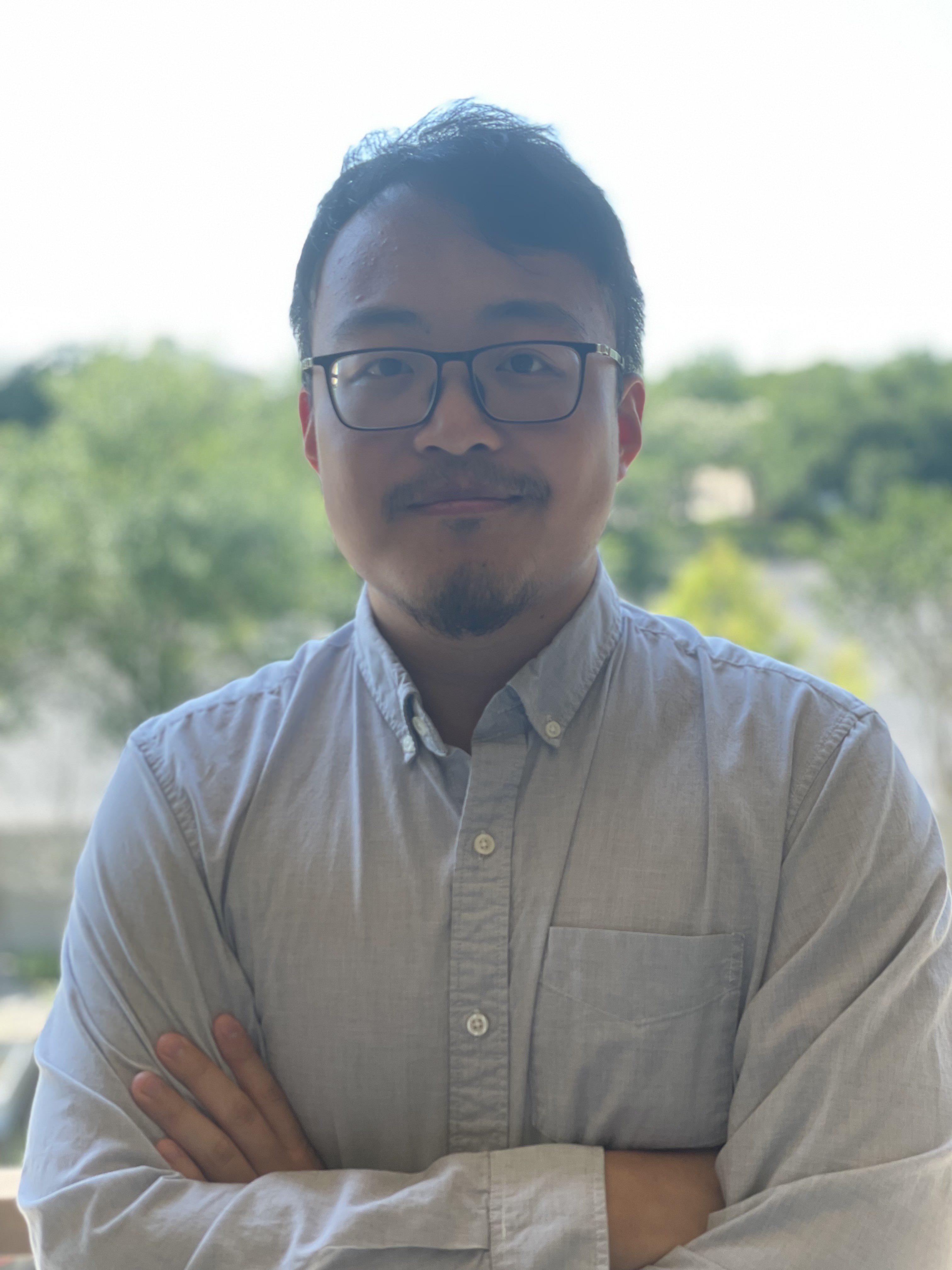}}] {Yinsong Wang} received his B.S. degree in Mechanical Engineering from Shandong University, China, in 2017, his M.S. degree in Manufacturing System Engineering and Management from The Hong Kong Polytechnic University, Hong Kong, in 2018. and his Ph.D. degree in Industrial Engineering at Northeastern University, USA, in 2024. He is currently an Assistant Professor of Industrial Engineering at the University of Nevada Reno. His research interests include machine learning, data science, and kernel methods.
	\end{IEEEbiography}

\begin{IEEEbiography}
	[{\includegraphics[width=1in,height=1.25in,clip,keepaspectratio]{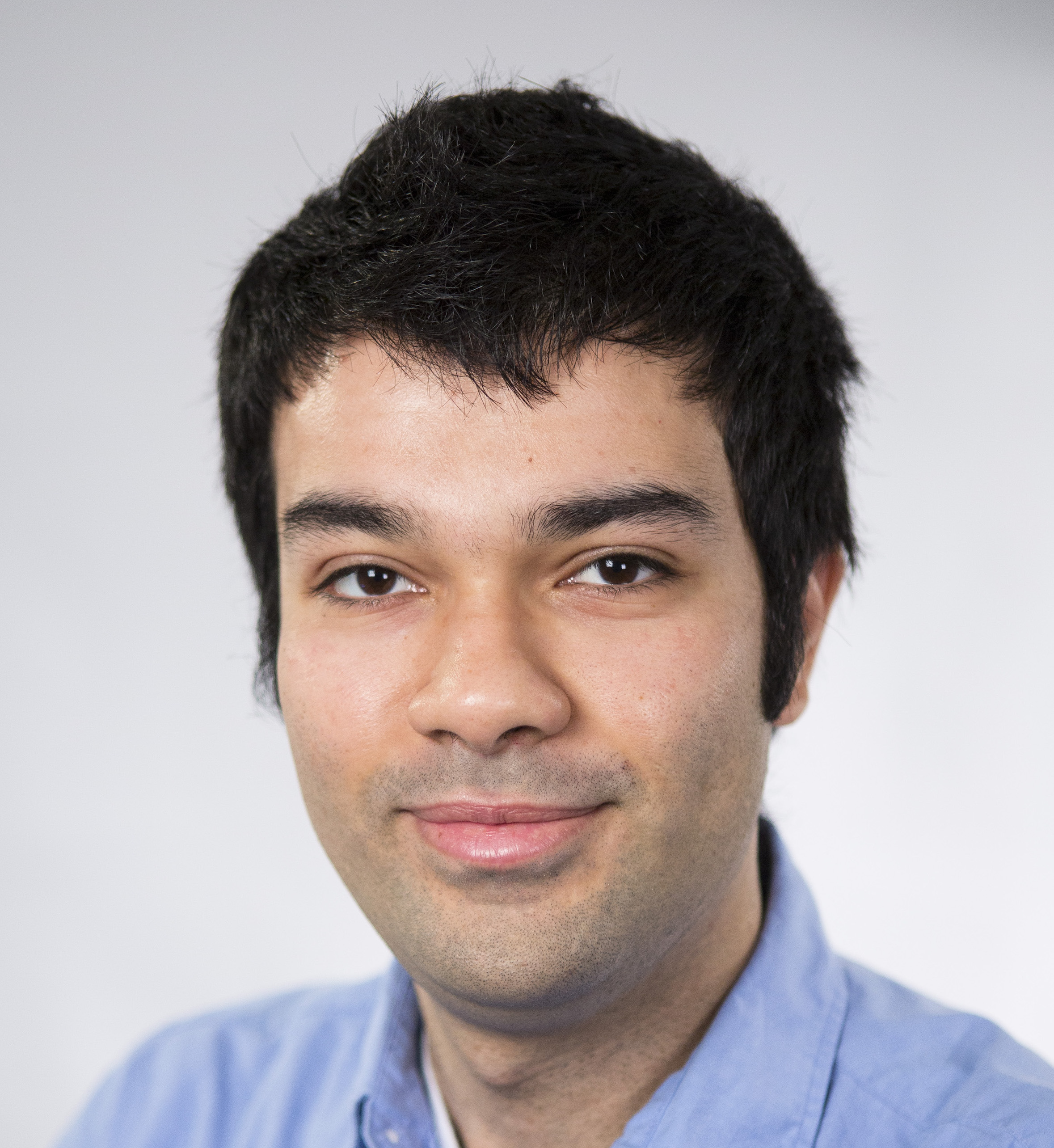}}] {Shahin Shahrampour} received the Ph.D. degree in Electrical and Systems Engineering, the M.A. degree in Statistics (The Wharton School), and the M.S.E. degree in Electrical Engineering, all from the University of Pennsylvania, in 2015, 2014, and 2012, respectively. He is currently an Associate Professor in the Department of Mechanical and Industrial Engineering at Northeastern University. His research interests include machine learning, optimization, sequential decision-making, and distributed learning, with a focus on developing computationally efficient methods for data analytics. He is a Senior Member of the IEEE.
	\end{IEEEbiography}

\clearpage

\section*{Appendix}

{\color{black}
\subsubsection*{Identification by Joint-Measure Matching}

Suppose the imputation problem is well specified, such that
$\zb=f_0(\xb)$ almost surely, and $\Gc_{nn}$ is discriminative. If
$R(f)=0$, then $d_{\Gc_{nn}}(\pib,\pib^f)=0$ implies
$\pib=\pib^f$. Equality of these joint measures implies equality of
their conditional measures given $\xb$, and hence
$\delta_{f_0(\xb)}=\delta_{f(\xb)}$,  $\pib(\xb)$-almost surely. Therefore,
\[
R(f)=0
\quad\Longleftrightarrow\quad
f(\xb)=f_0(\xb)
\quad \pib(\xb)\text{-almost surely}.
\]

\subsubsection*{Comparison with Generic Unlabeled-data Penalties}
Let $\Xc=\Zc=\{0,1\}$,
$\Db_l=\{(0,0)\}$, and $\Db_u=\{1\}$, with the true target satisfying
$f_0(1)=0$. Consider two deterministic predictors
\[
f_0(0)=f_1(0)=0,
\qquad
f_0(1)=0,
\qquad
f_1(1)=1.
\]
Both predictors interpolate $\Db_l$ and, when interpreted as
degenerate class probabilities, have zero prediction entropy. They can
also have identical zero pointwise consistency loss when their
predictions remain invariant under the prescribed perturbations.
Nevertheless, under the Euclidean ground metric on $(\xb,\zb)$,
\[
W_1(\pib_l,\pib_u^{f_0})=1,
\qquad
W_1(\pib_l,\pib_u^{f_1})=\sqrt{2}.
\]
Thus, joint-measure matching distinguishes the correct from the incorrect imputation in this example, whereas entropy and pointwise consistency alone do not. This example does not establish universal superiority of
MCR; it only illustrates that these penalties encode different
structural assumptions and do not automatically identify the joint
distribution targeted by MCR.
}

\subsection*{Omitted Proofs}
\subsubsection*{Proof of Theorem 1}

We decompose the estimation error for a prediction function $\hat{f}$ into the following three parts
    {
\setlength{\jot}{8pt}
\begin{equation}\label{eq:error_decomp}
    \begin{aligned}
        &d_{\Gc_{nn}}(\pib, \pib^{\hat{f}}) - \underset{f \in \Fc}{\inf} \;d_{\Gc_{nn}}(\pib, \pib^f) \\
        = & \underbrace{d_{\Gc_{nn}}(\pib,\pib^{\hat{f}}) - d_{\Gc_{nn}}(\pib_l,\pib^{\hat{f}})}_\text{\textcircled{1}}\\
        & + \underbrace{d_{\Gc_{nn}}(\pib_l,\pib^{\hat{f}}) - \underset{f \in \Fc}{\inf}\; d_{\Gc_{nn}}(\pib_l,\pib^f)}_\text{\textcircled{2}}\\
        & + \underbrace{\underset{f \in \Fc}{\inf}\; d_{\Gc_{nn}}(\pib_l,\pib^f) - \underset{f \in \Fc}{\inf} \;d_{\Gc_{nn}}(\pib, \pib^f)}_\text{\textcircled{3}}.
    \end{aligned}
\end{equation}}
We now need Lemma \ref{lemma:rademacher} to proceed with our bound.

\begin{lemma}\label{lemma:rademacher}
    For functions $h \in \Hc$ with maximum output size $B_h$, denoting the true measure of data as $\pib$ and the empirical measure of $n$ data points as $\pib_l$, we can bound the following quantity with probability at least $1-2\delta$
    \begin{equation*}
        \underset{h \in \Hc}{\sup} \; \big[ \pib h - \pib_l h \big] \leq 2\hat{\mathfrak{R}}_n(\Hc) + \frac{3B_h\sqrt{\ln \delta^{-1}}}{\sqrt{2n}},
    \end{equation*}
    where $\hat{\mathfrak{R}}_n(\Hc)$ denotes the Rademacher complexity of the function class $\Hc$ over $n$ data points, defined as 
\begin{equation}\label{eq:rademacher}
        \hat{\mathfrak{R}}_n(\Hc) := {\E}_\sigma\left[ \underset{h \in \Hc}{\sup} \frac{1}{n} \sum_{i=1}^n \sigma_i h(\xb_i) \right],
    \end{equation}
    where $\{\sigma_i\}_{i=1}^n$ are i.i.d. random variables uniformly sampled from $\{\pm 1\}$.
\end{lemma}

Term \textcircled{1} is easy to bound by rewriting the formulation for any prediction $f \in \Fc$ as follows
{
\setlength{\jot}{8pt}
\begin{equation}\label{eq:1and3}
\begin{aligned}
    &d_{\Gc_{nn}}(\pib,\pib^f) - d_{\Gc_{nn}}(\pib_l,\pib^f)\\
    = & \underset{g \in \Gc_{nn}}{\sup} \; \big[ \pib g - \pib^f g \big] - \underset{g \in \Gc_{nn}}{\sup} \; \big[ \pib_l g - \pib^f g \big]\\
    \leq & \underset{g \in \Gc_{nn}}{\sup} \; \big[ \pib g - \pib_l g \big].
\end{aligned}
\end{equation}}
Note that we rely on the symmetry of $\Gc_{nn}$ to switch between $\underset{g \in \Gc_{nn}}{\sup} \; \big[ \pib g - \pib^f g \big]$ and $\underset{g \in \Gc_{nn}}{\sup} \; \big| \pib g - \pib^f g \big|$. 

Let $f' \in \underset{f \in \Fc}{\arginf}\{ d_{\Gc_{nn}} (\pib,\pib^f)\}$. Then, for Term \textcircled{3}, we have
\begin{equation*}
    \begin{aligned}
        &\underset{f \in \Fc}{\inf}\; d_{\Gc_{nn}}(\pib_l,\pib^f) - \underset{f \in \Fc}{\inf} \;d_{\Gc_{nn}}(\pib, \pib^f) \\
        &\leq d_{\Gc_{nn}}(\pib_l,\pib^{f'}) - d_{\Gc_{nn}}(\pib,\pib^{f'})\\
        & = \underset{g \in \Gc_{nn}}{\sup} \; \big[ \pib_l g - \pib^{f'} g \big] - \underset{g \in \Gc_{nn}}{\sup} \; \big[ \pib g - \pib^{f'} g \big]\\
    & \leq \underset{g \in \Gc_{nn}}{\sup} \; \big[ \pib_l g - \pib g \big] = \underset{g \in \Gc_{nn}}{\sup} \; \big[ \pib g - \pib_l g \big],
    \end{aligned}
\end{equation*}
where the last line holds for any symmetric $\Gc_{nn}$ (i.e., $g \in \Gc_{nn} \Rightarrow -g \in \Gc_{nn}$). Notice that the above bound is independent of $f$ and is only dependent on the fully observed dataset $\Db_l$ and the function class $\Gc_{nn}$. 

We next bound the upper bound term in \eqref{eq:1and3}. With Lemma \ref{lemma:rademacher}, we have the upper bound with probability at least $1-2\delta$ such that
\begin{equation*}
    \begin{aligned}
        \underset{g \in \Gc_{nn}}{\sup} \; \big[ \pib g - \pib_l g \big] \leq 2\hat{\mathfrak{R}}_n(\Gc_{nn})+ \frac{3B_g B_{xz} \sqrt{\ln{\delta^{-1}}}}{\sqrt{2n}},
    \end{aligned}
\end{equation*}
and therefore, we conclude that
\begin{equation}\label{eq:bound1+3}
    \textcircled{1} + \textcircled{3} \leq 4\hat{\mathfrak{R}}_n(\Gc_{nn}) + \frac{3B_g B_{xz} \sqrt{2\ln{\delta^{-1}}}}{\sqrt{n}}.
\end{equation}

 The major difference between a prediction function learned with MCR versus one learned without MCR comes down to the difference in Term \textcircled{2}.

{\color{black} For the predictor $f_l^*$ learned without MCR, Assumption~\ref{assump:2}
implies that $f_l^*\in\Fc^*$, and hence
$\pib_l^{f_l^*}=\pib_l$. Therefore,
\begin{equation}
\label{eq:ideal_vanilla_term2}
\begin{aligned}
\textcircled{2}
&=
d_{\Gc_{nn}}(\pib_l,\pib^{f_l^*})
-
\inf_{f\in\Fc}d_{\Gc_{nn}}(\pib_l,\pib^f)
\\
&\leq
d_{\Gc_{nn}}(\pib_l,\pib^{f_l^*})
\\
&=
d_{\Gc_{nn}}(\pib_l^{f_l^*},\pib^{f_l^*})
\\
&\leq
\sup_{h\in\Hc}
\left|
\pib_l(\xb)h-\pib(\xb)h
\right|,
\end{aligned}
\end{equation}
where $\pib(\xb)$ and $\pib_l(\xb)$ denote the true (marginal) and empirical measure of $\xb$, respectively. Applying Lemma~\ref{lemma:rademacher} gives, with probability at least $1-2\delta$,
\begin{equation}
\label{eq:ideal_vanilla_term2_bound}
\textcircled{2}
\leq
2\hat{\mathfrak{R}}_n(\Hc)
+
\frac{\Delta}{2\sqrt n}.
\end{equation}
}

To proceed, we need to prove the following lemma.

\begin{lemma}\label{lemma}
    Recall the sets $\Fc_2$, $\Fc_3$, and $\Fc_4$ defined in Assumptions \ref{assump:2}, \ref{assump:3}, and \ref{assump:4}. Then, {\color{black} under Assumptions~\ref{assump:1}-\ref{assump:4}}, if $f^*_u \in \Fc_2 \cap \Fc_3$, we have $f^*_u \in \Fc_4$.
\end{lemma}

\noindent
{\bf Proof of Lemma \ref{lemma}:}
Suppose that there exists $f' \in \Fc_2 \cap \Fc_3$ that is not in $\Fc_2 \cap \Fc_4$. This means there exists $f^*\in \Fc_2 \cap \Fc_4$ such that the following identity holds
\begin{equation}\label{eq:contrad1}
d_{\Gc_{nn}}(\pib_l,\pib_{ul}^{f^*})
<
d_{\Gc_{nn}}(\pib_l,\pib_{ul}^{f'}).
\end{equation}
Now, since both $f' \in \Fc_2$ and $f^* \in \Fc_2$, we have $\pib_l^{f^*}=\pib_l^{f'}=\pib_l$. Therefore, for any $f\in\Fc_2$,
\begin{equation}\label{eq:ul_reduction}
    \begin{aligned}
d_{\Gc_{nn}}(\pib_l,\pib_{ul}^{f})
&=
\sup_{g\in\Gc_{nn}}
\Big[
\pib_l g
-\frac{n}{m+n}\pib_l^{f} g
-\frac{m}{m+n}\pib_u^{f} g
\Big] \nonumber\\
&=
\sup_{g\in\Gc_{nn}}
\Big[
\pib_l g
-\frac{n}{m+n}\pib_l g
-\frac{m}{m+n}\pib_u^{f} g
\Big] \nonumber\\
&=
\frac{m}{m+n}\sup_{g\in\Gc_{nn}}
\big(\pib_l g-\pib_u^{f} g\big)\\
&=
\frac{m}{m+n}\, d_{\Gc_{nn}}(\pib_l,\pib_u^{f}).
\end{aligned}
\end{equation}
Applying above to $f^*~\text{and}~f'$ and using the strict inequality \eqref{eq:contrad1} implies
\begin{equation*}
    d_{\Gc_{nn}}(\pib_l,\pib_u^{f^*})
    <
    d_{\Gc_{nn}}(\pib_l,\pib_u^{f'}).
\end{equation*}
This contradicts $f'\in\arg\inf_{f\in\Fc} d_{\Gc_{nn}}(\pib_l,\pib_u^{f})$.
Hence, the assumption $f'\notin \Fc_2 \cap \Fc_4$ is false, and therefore $f'\in \Fc_4$.

\qed

{\color{black} For the MCR solution $f_u^*$, Assumptions~\ref{assump:2} and
\ref{assump:3} imply $f_u^*\in\Fc^*\cap\Fc_3$, and therefore
$\pib_l^{f_u^*}=\pib_l$. Consequently, the same argument as in Eq.~\eqref{eq:ideal_vanilla_term2} gives
\begin{equation}
\label{eq:ideal_mcr_n_bound}
\textcircled{2}
\leq
2\hat{\mathfrak{R}}_n(\Hc)
+
\frac{\Delta}{2\sqrt n},
\end{equation}
with probability at least $1-2\delta$.

We can obtain a second bound by exploiting the mixture-compatibility condition. By Lemma~\ref{lemma}, we have $f_u^*\in\Fc_4$. Let
\[
\Tilde f
\in
\arg\inf_{f\in\Fc}
d_{\Gc_{nn}}(\pib_l,\pib^f).
\]
Since $f_u^*$ minimizes
$d_{\Gc_{nn}}(\pib_l,\pib_{ul}^{f})$, we have
\begin{align*}
\textcircled{2}
&=
d_{\Gc_{nn}}(\pib_l,\pib^{f_u^*})
-
d_{\Gc_{nn}}(\pib_l,\pib^{\Tilde f})
\\
&=
\Big[
d_{\Gc_{nn}}(\pib_l,\pib^{f_u^*})
-
d_{\Gc_{nn}}(\pib_l,\pib_{ul}^{f_u^*})
\Big]
\\
&\quad+
\Big[
d_{\Gc_{nn}}(\pib_l,\pib_{ul}^{f_u^*})
-
d_{\Gc_{nn}}(\pib_l,\pib_{ul}^{\Tilde f})
\Big]
\\
&\quad+
\Big[
d_{\Gc_{nn}}(\pib_l,\pib_{ul}^{\Tilde f})
-
d_{\Gc_{nn}}(\pib_l,\pib^{\Tilde f})
\Big]
\\
&\leq
2
\sup_{h\in\Hc}
\left|
\pib_{ul}(\xb)h-\pib(\xb)h
\right|.
\end{align*}
Hence, with probability at least $1-2\delta$,
\begin{equation}
\label{eq:ideal_mcr_mixed_bound}
\textcircled{2}
\leq
4\hat{\mathfrak{R}}_{m+n}(\Hc)
+
\frac{\Delta}{\sqrt{m+n}}.
\end{equation}

On the intersection of the two concentration events,
Eqs.~\eqref{eq:ideal_mcr_n_bound} and
\eqref{eq:ideal_mcr_mixed_bound} yield
\begin{equation}
\label{eq:ideal_mcr_min_bound}
\textcircled{2}
\leq
\min\left\{
2\hat{\mathfrak{R}}_n(\Hc)
+\frac{\Delta}{2\sqrt n},
\;
4\hat{\mathfrak{R}}_{m+n}(\Hc)
+\frac{\Delta}{\sqrt{m+n}}
\right\}.
\end{equation}

Combining Eq.~\eqref{eq:ideal_vanilla_term2_bound} and
Eq.~\eqref{eq:ideal_mcr_min_bound}, respectively, with
Eq.~\eqref{eq:bound1+3}, and applying a union bound proves Theorem~\ref{the:main}.
}
\qed

\begin{table*}[h]
    \centering
    \caption{Ratio Test HyperParameter}
    \begin{tabular}{|c|c|c|c|c|c|c|c|c|c|}
        \hline
        $\alpha_1$ & $\alpha_2$ & $\alpha_3$ & $\beta_1$ & $\beta_2$ & $\lambda_w$ & $\lambda_{\nabla}$ & $\lambda_d$ & $n_f$ & $n_g$ \\
        \hline
        $0.002$ & $0.0015$ & $0.0015$ & $0.9$ & $0.999$ & $0.01$ & $10$ & $0.01$ & $4000$ & $1$ \\
        \hline
    \end{tabular}
    \vspace{5pt}
    \label{tab:ratio_para}
\end{table*}

\subsubsection*{Proof of Theorem \ref{the:imperfect}}

\begin{table*}[h]
    \centering
    \caption{Capacity test hyperparameters, the first block are hyperparameters for varying width test, and the second block are hyperparameters for varying depth test. The left column shows the specific test setting (value for width or depth).}
    \begin{tabular}{|c||c|c|c|c|c|c|c|c|c|c|}
        \hline
         &$\alpha_1$ & $\alpha_2$ & $\alpha_3$ & $\beta_1$ & $\beta_2$ & $\lambda_w$ & $\lambda_{\nabla}$ & $\lambda_d$ & $n_f$ & $n_g$ \\
        \hline
        \hline
        $16$ & $0.005$ & $0.005$ & $0.0015$ & $0.9$ & $0.999$ & $0.01$ & $10$ & $0.01$ & $4000$ & $1$ \\
        \hline
        $32$ & $0.0025$ & $0.0025$ & $0.0015$ & $0.9$ & $0.999$ & $0.01$ & $10$ & $0.01$ & $4000$ & $1$ \\
        \hline
        $64$ & $0.002$ & $0.0015$ & $0.0015$ & $0.9$ & $0.999$ & $0.01$ & $10$ & $0.01$ & $4000$ & $1$ \\
        \hline
        $128$ & $0.0015$ & $0.0015$ & $0.0015$ & $0.9$ & $0.99$ & $0.01$ & $10$ & $0.01$ & $4000$ & $1$ \\
        \hline
        $256$ & $0.0011$ & $0.0011$ & $0.0015$ & $0.9$ & $0.95$ & $0.01$ & $10$ & $0.01$ & $4000$ & $1$ \\
        \hline
        \hline
        $2$ & $0.005$ & $0.005$ & $0.0015$ & $0.9$ & $0.999$ & $0.01$ & $10$ & $0.01$ & $4000$ & $1$ \\
        \hline
        $4$ & $0.001$ & $0.001$ & $0.0015$ & $0.9$ & $0.999$ & $0.01$ & $10$ & $0.01$ & $4000$ & $1$ \\
        \hline
        $6$ & $0.0005$ & $0.0005$ & $0.0015$ & $0.9$ & $0.999$ & $0.01$ & $10$ & $0.01$ & $4000$ & $1$ \\
        \hline
    \end{tabular}
    \vspace{5pt}
    \label{tab:capacity_para}
\end{table*}

\begin{table*}[h]
    \centering
    \caption{Duality Gap Stopping Condition Evaluation HyperParameter}
    \begin{tabular}{|c|c|c|c|c|c|c|c|c|c|}
        \hline
        $\alpha_1$ & $\alpha_2$ & $\alpha_3$ & $\beta_1$ & $\beta_2$ & $\lambda_w$ & $\lambda_{\nabla}$ & $\lambda_d$ & $n_f$ & $n_g$ \\
        \hline
        $0.005$ & $0.005$ & $0.005$ & $0.9$ & $0.999$ & $0.01$ & $10$ & $0.01$ & $2000$ & $1$ \\
        \hline
    \end{tabular}
    \vspace{5pt}
    \label{tab:dg_hype}
\end{table*}

\begin{table*}[h]
    \centering
    \caption{MCR Weight $\lambda_d$ Evaluation HyperParameter}
    \begin{tabular}{|c|c|c|c|c|c|c|c|c|c|}
        \hline
        $\alpha_1$ & $\alpha_2$ & $\alpha_3$ & $\beta_1$ & $\beta_2$ & $\lambda_w$ & $\lambda_{\nabla}$& $n_f$ & $n_g$ \\
        \hline
        $0.005$ & $0.005$ & $0.0015$ & $0.9$ & $0.99$ & $0$ & $10$ & $2000$ & $1$ \\
        \hline
    \end{tabular}
    \vspace{5pt}
    \label{tab:lambdad}
\end{table*}

\begin{table*}[h]
    \centering
    \caption{Neural Net Distance Evaluation HyperParameter}
    \begin{tabular}{|c|c|c|c|c|c|c|c|c|c|}
        \hline
        $\alpha_1$ & $\alpha_2$ & $\alpha_3$ & $\beta_1$ & $\beta_2$ & $\lambda_w$ & $\lambda_{\nabla}$ & $n_f$ & $n_g$ \\
        \hline
        $0.005$ & $0.002$ & $0.0015$ & $0.9$ & $0.99$ & $0$ & $10$ & $2000$ & $1$ \\
        \hline
    \end{tabular}
    \vspace{5pt}
    \label{tab:nnd}
\end{table*}

\begin{table*}[h]
    \centering
    \caption{Distribution Discrepancy Evaluation HyperParameter}
    \begin{tabular}{|c|c|c|c|c|c|c|c|c|c|}
        \hline
        $\alpha_1$ & $\alpha_2$ & $\alpha_3$ & $\beta_1$ & $\beta_2$ & $\lambda_w$ & $\lambda_{\nabla}$ & $\lambda_d$ & $n_f$ & $n_g$ \\
        \hline
        $0.005$ & $0.005$ & $0.0015$ & $0.9$ & $0.99$ & $0$ & $10$ & $0.01$ & $2000$ & $1$ \\
        \hline
    \end{tabular}
    \vspace{5pt}
    \label{tab:xi}
\end{table*}

     Notice that the error decomposition in Eq. \eqref{eq:error_decomp} breaks down the estimation error into Terms $\textcircled{1}$, $\textcircled{2}$, and $\textcircled{3}$, where Terms $\textcircled{1}$ and $\textcircled{3}$ can be upper bounded independently for any learned prediction function. Therefore, the same upper bound holds, and we only need to focus on Term $\textcircled{2}$ in this section. {\color{black} For vanilla training and the resulting solution $f_l$, Term~\textcircled{2} satisfies
\begin{align}
\textcircled{2}
&=
d_{\Gc_{nn}}(\pib_l,\pib^{f_l})
-
\inf_{f\in\Fc}d_{\Gc_{nn}}(\pib_l,\pib^f)
\nonumber\\
&\leq
d_{\Gc_{nn}}(\pib_l,\pib^{f_l})
\nonumber\\
&\leq
d_{\Gc_{nn}}(\pib_l,\pib_l^{f_l})
+
d_{\Gc_{nn}}(\pib_l^{f_l},\pib^{f_l})
\nonumber\\
&\leq
C(\epsilon_{\Lc})^\alpha
+
\sup_{h\in\Hc}
\left|
\pib_l(\xb)h-\pib(\xb)h
\right|,
\label{eq:imperfect_vanilla_term2}
\end{align}
where we used Assumption \ref{assump:5} in the last line. Therefore, with probability at least $1-2\delta$,
\begin{equation}
\textcircled{2}
\leq
C(\epsilon_{\Lc})^\alpha
+
2\hat{\mathfrak{R}}_n(\Hc)
+
\frac{\Delta}{2\sqrt n}.
\end{equation}
Combining above with Eq.~\eqref{eq:bound1+3} proves
Eq.~\eqref{eq:imperfect_loss}.

\begin{table*}[h]
    \centering
    \caption{Real-data imputation hyperparameters used in the final experiments.
    In Panel A, entries for $n_f$ are reported as vanilla/MCR-method epochs.
    The learning rates $\alpha_1$, $\alpha_2$, and $\alpha_3$ retain their definitions
    from the algorithm: $\alpha_1$ is used for vanilla training, $\alpha_2$ for
    predictor training with a distributional penalty, and $\alpha_3$ for the MCR
    critic or GAN discriminator. Panel B gives the method-specific penalty weights.
    Panel C reports additional penalty settings shared across datasets.}
    
    \small
    \setlength{\tabcolsep}{4pt}
    \renewcommand{\arraystretch}{1.1}

    \begin{tabular}{|c|c|c|c|c|c|c|c|c|c|c|}
        \hline
        \multicolumn{11}{|c|}{\textbf{Panel A: Dataset-level optimization parameters}} \\
        \hline
        Dataset
        & $\alpha_1$
        & $\alpha_2$
        & $\alpha_3$
        & $\beta_1$
        & $\beta_2$
        & $\lambda_w$
        & $\lambda_{\nabla}$
        & $n_f$
        & $n_g$
        & $B$ \\
        \hline
        MNIST
        & $0.0001$
        & $0.0001$
        & $0.00001$
        & $0.9$
        & $0.999$
        & $0.01$
        & $10$
        & $500/500$
        & $2$
        & $500$ \\
        \hline
        CIFAR-10
        & $0.00001$
        & $0.00001$
        & $0.001$
        & $0.9$
        & $0.99$
        & $0.01$
        & $10$
        & $500/500$
        & $3$
        & $500$ \\
        \hline
        Activity
        & $0.0005$
        & $0.0005$
        & $0.0001$
        & $0.9$
        & $0.999$
        & $0.01$
        & $10$
        & $500/750$
        & $2$
        & $200$ \\
        \hline
        Crop
        & $0.0003$
        & $0.0003$
        & $0.001$
        & $0.9$
        & $0.999$
        & $0.01$
        & $10$
        & $500/500$
        & $2$
        & $500$ \\
        \hline
        scMultiSim
        & $0.0001$
        & $0.0001$
        & $0.00001$
        & $0.9$
        & $0.95$
        & $0.01$
        & $10$
        & $500/750$
        & $5$
        & $100$ \\
        \hline
    \end{tabular}

    \vspace{6pt}

    \begin{tabular}{|c|c|c|c|c|c|}
        \hline
        \multicolumn{6}{|c|}{\textbf{Panel B: Dataset- and method-specific penalty weights}} \\
        \hline
        Dataset
        & $\lambda_d$
        & $\lambda_{\mathrm{GAN}}$
        & $\lambda_{\mathrm{MMD}}$
        & $\lambda_{\mathrm{SH}}$
        & $\lambda_{W_2}$ \\
        \hline
        MNIST
        & $0.001$
        & $0.001$
        & $0.01$
        & $0.1$
        & $0.1$ \\
        \hline
        CIFAR-10
        & $0.002$
        & $0.001$
        & $0.01$
        & $0.01$
        & $0.01$ \\
        \hline
        Activity
        & $0.001$
        & $0.001$
        & $1$
        & $0.1$
        & $0.1$ \\
        \hline
        Crop
        & $0.001$
        & $0.001$
        & $0.01$
        & $0.1$
        & $0.1$ \\
        \hline
        scMultiSim
        & $0.0001$
        & $0.1$
        & $1$
        & $1$
        & $1$ \\
        \hline
    \end{tabular}

    \vspace{6pt}

    \begin{tabular}{|c|c|c|}
        \hline
        \multicolumn{3}{|c|}{\textbf{Panel C: Method-specific penalty settings}} \\
        \hline
        Method
        & Penalty/adversary batch size
        & Additional setting \\
        \hline
        Vanilla
        & -- 
        & No distributional penalty \\
        \hline
        MCR-W
        & $B$
        & WGAN-GP; $n_g$ and $\lambda_{\nabla}$ are given in Panel A \\
        \hline
        GAIN
        & $B$
        & One discriminator update per predictor update \\
        \hline
        GAIN-MMD
        & $100$
        & Median-heuristic bandwidth with scale $s_{\sigma}=1$ \\
        \hline
        OT-Sinkhorn
        & $100$
        & $s_{\varepsilon}=0.05$ and $30$ Sinkhorn iterations \\
        \hline
        OT-$W_2$
        & $100$
        & Equal-weight empirical transport via Hungarian assignment \\
        \hline
    \end{tabular}

    \vspace{5pt}
    \label{tab:real_data_para}
    \label{tab:image_para}
    \label{tab:sensor_para}
    \label{tab:scvi_para}
\end{table*}

For MCR training and the resulting solution $f_u$,
\begin{align}
\textcircled{2}
&=
d_{\Gc_{nn}}(\pib_l,\pib^{f_u})
-
\inf_{f\in\Fc}d_{\Gc_{nn}}(\pib_l,\pib^f)
\nonumber\\
&\leq
d_{\Gc_{nn}}(\pib_l,\pib^{f_u})
\nonumber\\
&\leq
d_{\Gc_{nn}}(\pib_l,\pib_{ul}^{f_u})
+
d_{\Gc_{nn}}(\pib_{ul}^{f_u},\pib^{f_u}).
\label{eq:imperfect_mcr_term2_new}
\end{align}

Using the definition of $\pib_{ul}^{f_u}$ and convexity of an IPM under mixtures, we have
\begin{equation}
\resizebox{\columnwidth}{!}{$
\begin{aligned}
d_{\Gc_{nn}}(\pib_l,\pib_{ul}^{f_u})
&\leq
\frac{n}{m+n}
d_{\Gc_{nn}}(\pib_l,\pib_l^{f_u})
+
\frac{m}{m+n}
d_{\Gc_{nn}}(\pib_l,\pib_u^{f_u})
\nonumber\\
&\leq
\frac{n}{m+n}C(\epsilon_{\Lc})^\alpha
+
\frac{m}{m+n}(\epsilon_d+\xi),
\label{eq:imperfect_mcr_mixture}
\end{aligned}$}
\end{equation}
where we used Assumptions~\ref{assump:5} and~\ref{assump:6} as well as $\xi:=
\inf_{f\in\Fc}
d_{\Gc_{nn}}(\pib_l,\pib_u^f)$ in the last inequality. We also have
\begin{align}
d_{\Gc_{nn}}(\pib_{ul}^{f_u},\pib^{f_u})
&\leq
\sup_{h\in\Hc}
\left|
\pib_{ul}(\xb)h-\pib(\xb)h
\right|.
\end{align}
Hence, with probability at least $1-2\delta$,
\begin{align}
\textcircled{2}
\leq{}&
2\hat{\mathfrak{R}}_{m+n}(\Hc)
+
\frac{\Delta}{2\sqrt{m+n}}
\nonumber\\
&+
\frac{n}{m+n}C(\epsilon_{\Lc})^\alpha
+
\frac{m}{m+n}(\epsilon_d+\xi).
\end{align}

Combining this result with Eq.~\eqref{eq:bound1+3}
proves Eq.~\eqref{eq:imperfect_ndp}.}

\qed

\subsection*{Simulation Settings}

Throughout this section, we will use consistent notation for all optimization hyperparameters. We use $\alpha_1$ for the learning rate of the model trained without MCR. $\alpha_2$ for learning rate of the model trained with MCR. $\alpha_3$ for the learning rate of the neural distance function $g$. $\beta_1$, $\beta_2$, and $\lambda_w$ are the parameters of the AdamW algorithm. $\lambda_{\nabla}$ is the weight for the gradient penalty. $\lambda_d$ is the weight of the neural distance penalty. The corresponding penalty weights for the benchmark algorithms are, $\lambda_{GAN}$ for GAIN, $\lambda_{MMD}$ for GAIN-MMD, $\lambda_{SH}$ for OT-Sinkhorn, and $\lambda_{W_2}$ for OT-$W_2$. $n_f$ for the maximum epoch number for updating the prediction function $f$, $n_g$ for the epoch number for updating neural distance function $g$.

\subsubsection*{Verifying Theoretical Claims}

For comparison of the $m$ and $n$ ratios, we set $d_x = 30, d_z = 20$. The generation process is achieved with randomly initialized two-layer ReLU activation MLP with $64$ neurons in the first hidden layer and $32$ in the second hidden layer. The prediction function is also a two-hidden-layer MLP with ReLU activation and batch normalization; both hidden layers have $64$ neurons. The neural distance function is another two-hidden-layer ReLU activation MLP with $512$ neurons in the first hidden layer and $256$ neurons in the second hidden layer.

In the learning process, the primary loss function $\Lc$ is the mean squared error loss, and the optimizer is AdamW. The hyperparameter settings are shown in Table \ref{tab:ratio_para}.

For model capacity comparison, we follow the same data generation process as described earlier. However, we always generate $100$ fully observed training data, $1000$ partially observed data for penalty calculation, and $2000$ fully observed test data. For varying the model width, we set the depth of the MLP prediction function to $2$ (two hidden layers). For varying the depth of the model, we set the width of the model to $64$. We also use the same neural distance function. The hyperparameter settings are shown in Table \ref{tab:capacity_para}.

For MCR weight $\lambda_d$ comparison, we set $m=2000$ and $n=100$, and the parameter settings are shown in Table~\ref{tab:lambdad}.

For choice of neural net distance $\Gc_{nn}$, the hyperparameters settings are shown in Table~\ref{tab:nnd}. Note that the only difference between MMD and $1-$Wasserstein MCR hyperparameters settings is the $\lambda_d$ values, in consideration of the loss scales to ensure the primary MSE loss converges to similar values in both scenarios.

For the distribution discrepancy $\hat{\xi}$ evaluation, the hyperparameters settings are shown in Table~\ref{tab:xi}.

\subsubsection*{Evaluating Duality Gap Stopping Condition}

We use the same synthetic data generation process as in the previous section, with $n$ set to $100$ and $m$ set to $2000$. The hyper parameters settings are shown in Table~\ref{tab:dg_hype}. For duality gap estimation, we set the duality gap update step to $n_d = 5$ for a more accurate estimation.

\subsubsection*{Versatility of MCR}

For the MNIST reconstruction experiment, we normalize the data to $[0,1]$ by dividing it by $255$. For the CIFAR10 dataset, we do the same thing, while cropping out along the pixel dimension. For both datasets, we choose $1000$ data points as fully observed training data, $5000$ cropped images as penalty data, and the rest as test data. For the prediction function $f$, we choose a convolutional autoencoder with convolution filters in the encoder and convolution transpose in the decoder. The encoder and decoder have a mirror structure, with $32$ channels in the first layer and $64$ channels in the second layer. The latent dimension is $32$ for MNIST and $128$ for CIFAR. We use the ReLU activation function. The neural distance function now has convolution filters with $16$ channels, followed by a linear layer. We also use the AdamW optimizer. The hyperparameter setting is shown in Table \ref{tab:image_para}, where the first row is for MNIST and the second row is for CIFAR10.

For the Activity dataset, we standardize the data set to mean $0$ and standard deviation $1$ along each feature dimension. We sample $400$ data points as fully observed training data, and $2000$ TEB and EDA data for neural distance penalty. The rest of the dataset are used as test data. For the Crop dataset, we do the same standardization procedure, and we sample $1000$ data as fully observed training data, and $5000$ radar data for neural distance penalty. For the prediction function $f$, we use variational autoencoders. Again, encoder and decoder have mirror structures, with the first layer of width $512$, and second layer of width $256$ for both datasets, and a latent dimension of $128$ for Activity dataset, and $32$ for Crop dataset. We use the same MLP neural distance function $g$ as before. The optimization hyperparameter setting is shown in Table \ref{tab:sensor_para}. First row is for Activity dataset, and the second row is for Crop dataset.

Due to a lack of data points in the scMultiSim dataset, we only divide the datasets into two parts. The first part is $200$ fully observed training data, and the second part contains $300$ RNA profiles for neural distance penalty calculation and $300$ ATAC-seq data for testing. We follow the same data normalization approach in \cite{cohen2023joint}. For the prediction function $f$, we write a custom version of MultiVI \cite{ashuach2023multivi} by removing the module that encodes library size, since we only have access to two data modalities. Moreover, we set the weighting factor for KL divergence in the latent space to $10$. We use the same MLP neural net distance function $g$ as before as it still works well in this case. The hyperparameter setting is shown in Table~\ref{tab:scvi_para}.

We also note that the parameter $B$ in Table~\ref{tab:scvi_para} represents the penalty calculation batch size, which is different across datasets for computation feasibility.

\end{document}